\documentclass{article}

\usepackage{authblk}
\usepackage[square,numbers]{natbib}
\usepackage{graphicx}
\usepackage{amsmath,amssymb}
\usepackage{subcaption}
\usepackage{booktabs} 
\usepackage{hyperref}
\usepackage{url}
\usepackage{xcolor}
\usepackage{mathrsfs}
\usepackage{algorithm,algorithmic}
\usepackage{geometry}
\usepackage{cite}
\usepackage{bbm}
\usepackage{bm}

\usepackage[utf8]{inputenc} 
\usepackage[T1]{fontenc}    
\usepackage{hyperref}       
\usepackage{url}            
\usepackage{booktabs}       
\usepackage{amsfonts}       
\usepackage{nicefrac}       
\usepackage{microtype}      
\usepackage{xcolor}         
\usepackage{multirow}
\usepackage{colortbl}
\usepackage{wrapfig}
\usepackage{lipsum}

\graphicspath{{figs/pdf/}}
\renewcommand{\cal}{\mathcal}
\newcommand\cA{{\mathcal A}}

\newcommand{\cF}{{\cal F}}
\newcommand{\cY}{{\cal Y}}
\newcommand{\cD}{{\cal D}}

\newcommand\cH{{\mathcal H}}
\newcommand{\cK}{{\cal K}}
\newcommand{\cL}{{\cal L}}
\newcommand{\cM}{{\cal M}}
\newcommand{\cN}{{\cal N}}

\newcommand{\cP}{{\cal P}}

\newcommand{\cR}{{\mathcal R}}
\newcommand{\cS}{{\mathcal S}}

\newcommand{\cX}{{\mathcal X}}
\newcommand{\R}{\mathbb{R}}	%
\newcommand{\dist}{\operatorname{dist}}

\newcommand{\error}{\operatorname{err}}
\newcommand{\err}{\widehat{\error}}

\newcommand{\bE}{\mathbb{E}}

\newcommand{\bN}{\mathbb{N}}
\newcommand{\bP}{\mathbb{P}}

\newcommand{\bR}{{\mathbb R}}

\newcommand{\eps}{\epsilon}

\newcommand{\Lbal}{\cL_{\text{bal}}}
\newcommand{\Lcomb}{\cL_{\text{cb}}}
\newcommand{\Lfair}{\cL_{\text{fv}}}

\newcommand{\lfifa}{\ell_{\text{FIFA}}}
\newcommand{\gap}{\operatorname{gap}}
\newcommand{\abs}[1]{\left\lvert #1 \right\rvert}

\newenvironment{proof}{\textit{Proof:}}{\hfill$\square$}

\newtheorem{theorem}{Theorem}[section]

\newtheorem{lemma}{Lemma}[section]

\newtheorem{remark}{Remark}[section]

\newtheorem{example}{Example}[section]

\usepackage{algorithm, algorithmic}

\DeclareMathOperator*{\argmax}{arg\,max}

\def\cJ{\mathcal{J}}

\def\bal{\text{bal}}

\title{FIFA: Making Fairness More Generalizable in Classifiers Trained on Imbalanced Data}

\author[1]{Zhun Deng
\thanks{Corresponding authors: {\tt zhundeng@g.harvard.edu}, {\tt zjiayao@upenn.edu}.}}
\author[2]{Jiayao Zhang }
\author[3]{Linjun Zhang}
\author[4]{Ting Ye}
\author[4,5]{Yates Coley}
\author[2]{Weijie J.~Su}
\author[6]{James Zou}
\affil[1]{Harvard University}
\affil[2]{University of Pennsylvania}
\affil[3]{Rutgers University}
\affil[4]{University of Washington}
\affil[5]{Kaiser Permanente Washington Health Research Institute}
\affil[6]{Stanford University}
\date{}

\begin{document}

\maketitle

\begin{abstract}
Algorithmic fairness plays an important role in machine learning and imposing fairness constraints during learning is a common approach. However, many datasets are imbalanced in certain label classes (e.g. "healthy") and sensitive subgroups (e.g. "older patients"). Empirically, this imbalance leads to a lack of generalizability not only of classification, but also of fairness properties, especially in over-parameterized models. For example, fairness-aware training may 
ensure equalized odds (EO) on the training data, but EO is far from being satisfied on new users. In this paper, we propose a theoretically-principled, yet {\bf F}lexible approach that is {\bf I}mbalance-{\bf F}airness-{\bf A}ware ({\bf FIFA}). Specifically, FIFA encourages both classification and fairness generalization and can be flexibly combined with many existing fair learning methods with logits-based losses. While our main focus is on EO, FIFA can be directly applied to achieve equalized opportunity (EqOpt); and under certain conditions, it can also be applied to other fairness notions. We demonstrate the power of FIFA by combining it with a popular fair classification algorithm, and the resulting algorithm achieves significantly better fairness generalization on several real-world datasets. 

\end{abstract}

\section{Introduction}
\begin{wrapfigure}{r}{0.35\textwidth}
    \vspace{-14pt}
        \centering
        \includegraphics[width=\linewidth]{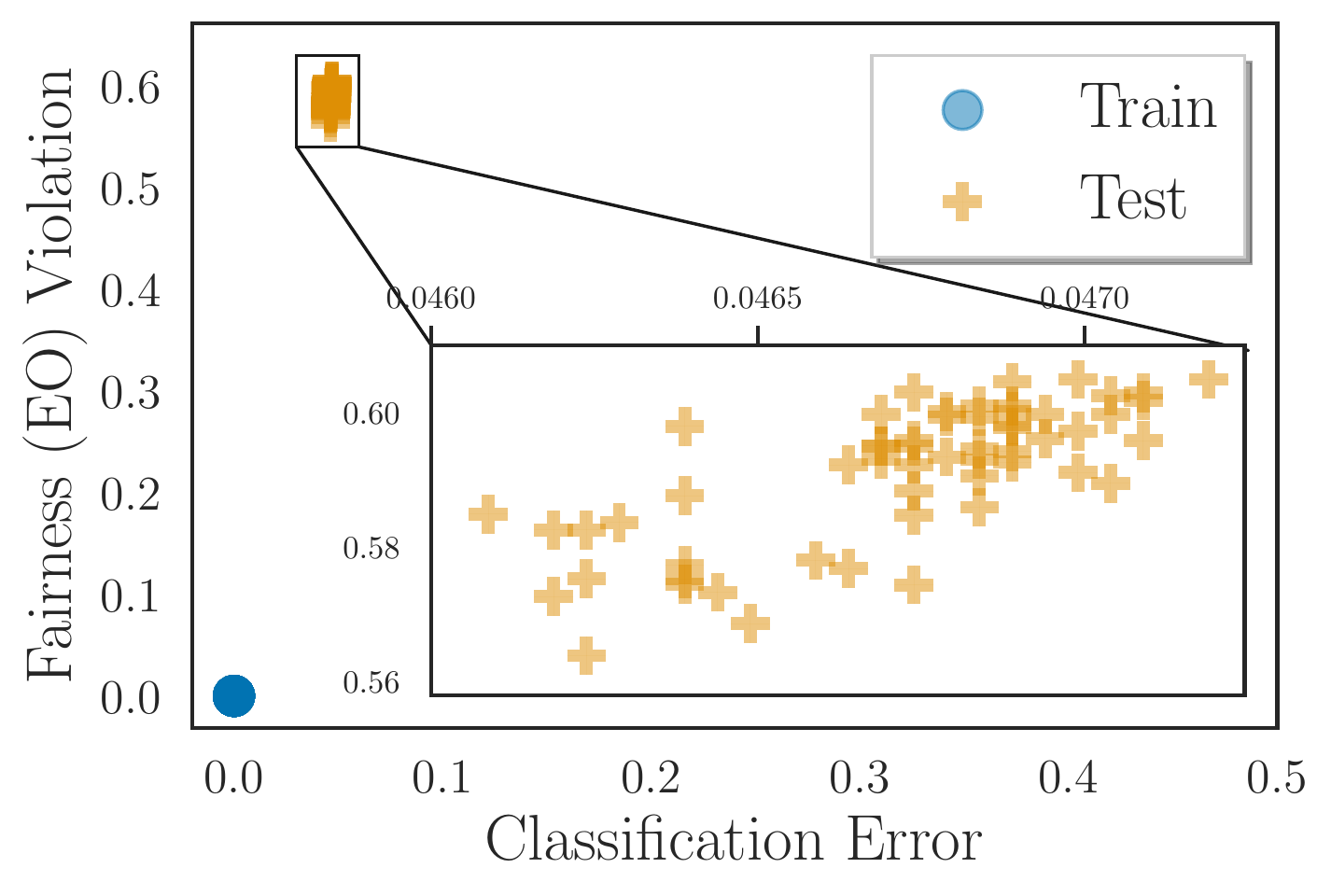}
    	\caption{Each marker corresponds to a sufficiently well-trained ResNet-10 model trained on an imbalanced image classification dataset CelebA (\citep{liu2015faceattributes}).
    	The generalization of fairness constraints (EO)
    	is of magnitudes poorer compared that on
    	classification error.
    	}
    	\label{fig:teaser}
    \vspace{-30pt}
\end{wrapfigure}

Machine learning systems are becoming increasingly vital in our daily lives. The 
growing concern
that they may inadvertently discriminate against minorities and other protected groups when identifying or allocating resources has 
attracted numerous attention from various communities both inside and outside academia.
While significant efforts have been devoted in understanding and correcting biases in classical models such as logistic regressions and supported vector machines (SVM), see, e.g., \citep{agarwal2018reductions,hardt2016equality}, those derived tools are far less effective on modern over-parameterized models such as neural networks (NN).
Over-parameterization
can lead to poor generalization,
as extensive efforts in both theoretical
and empirical studies have exhibited (\citep{zhang2017rethink,nakkiran2020deep,bartlett2020benign}).
Furthermore, in large models, it is also
difficult for \emph{measures of fairness} (such as equalized odds to be introduced shortly) to generalize, as shown in Fig.~\ref{fig:teaser}. Here we find that sufficiently trained ResNet-10 models generalize well on classification error but  poorly on fairness constraints---the gap in equalized odds between the test and training data is more than ten times larger than the gap for classification error between test and training. 

\begin{wrapfigure}{r}{0.35\textwidth}
    \vspace{-15pt}
        \centering
        \includegraphics[width=\linewidth]{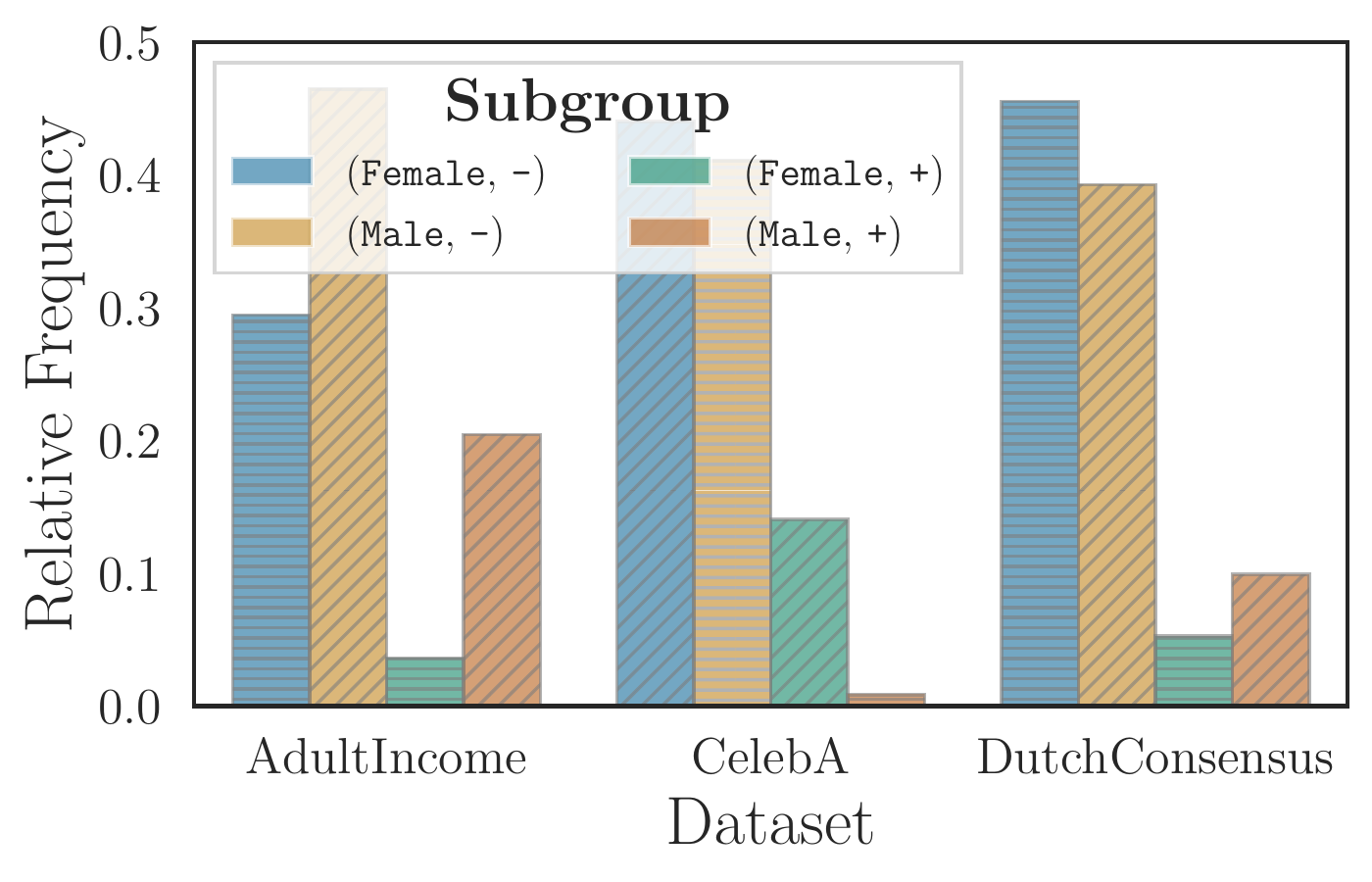}
    	\caption{The subgroups (sensitive attribute, either {\tt Male} or {\tt Female}, and label class, either {\tt +} or {\tt --}) are very
    	imbalanced in many popular datasets
    	across different domains.
    	}
    	\label{fig:teaser:hist}
    \vspace{-10pt}
\end{wrapfigure}%

In parallel, another outstanding challenge
for generalization with real-world datasets is that
they are often \emph{imbalanced} across label and demographic groups (see 
Fig.~\ref{fig:teaser:hist} 
for imbalance in
three commonly used datasets across various domains).
This inherent nature of real-world data,  greatly hinders the generalization of  classifiers that are unaware of this innate imbalance,
especially when the performance measure 
places substantial emphasis on minority classes or subgroups without sufficient samples (e.g., when considering the average classification error for each label class).
Although generalizations with imbalanced data has been extensively studied and
mitigation strategies are proposed
\citep{cao2019learning,mani2003knn,he2009learning,an2020resampling,he2013imbalanced,krawczyk2016learning}, it's unclear
how well fairness properties generalize. It is also an open challenge how to improve the generalization of fairness in over-parameterized models.

\paragraph{Our contributions.} In this paper, we initiate the study for this open challenge of fairness generalization for supervised classification tasks in imbalanced datasets. Inspired by recent works on regularizing the minority classes more strongly than the frequent classes by imposing class-dependent margins \citep{cao2019learning} in standard supervised learning, we design a theoretically-principled, {\bf F}lexible and {\bf I}mbalance-{\bf F}airness-{\bf A}ware (FIFA) approach that takes both classification error and fairness constraints violation into account when training the model. Our proposed method FIFA 
can be flexibly combined with many fair learning methods with logits-based losses
such as the soft margin loss \citep{liu2016large}
by encouraging larger margins for minority subgroups. 
While our method appears to be motivated for over-parameterized models such as neural networks, it nonetheless also helps simpler models such as logistic regressions. Experiments on both large datasets using over-parameterized models  as well as smaller
datasets using simpler models demonstrate the effectiveness,
and flexibility of our approach in ensuring a better fairness generalization while preserving good classification generalization.

\paragraph{Related work.}
Over-parameterized models such as neural network has achieved great performance in modern learning, however, the nature of optimization \citep{deng2020representation,deng2021adversarial,ji2021unconstrained,ji2021power,kawaguchi2022understanding,zhang2021and}, robustness \citep{deng2020interpreting,deng2021improving,deng2020towards}, and generalization \citep{deng2020towards,deng2021shrinking,zhang2020does,yao2021improving,yao2021meta} of neural networks remains mysterious. Even worse, the inherent imbalance in certain label classes bring great trouble for neural networks to generalize in supervised learning, which has attracted significant interest in the machine learning communities. Several methods including
resampling, reweighting, and data augmentation have been developed and deployed in practice \citep{mani2003knn,he2009learning,an2020resampling,he2013imbalanced,krawczyk2016learning,chang2017active,haixiang2017learning,sagawa2020investigation,wang2017learning,cui2019class,byrd2019effect}. Theoretical
analyses of those methods include margin-based approaches \citep{li2002perceptron,kakade2008complexity,khan2019striking,cao2019learning}.
Somewhat tangentially, an outstanding and emerging problem faced by modern models with real-world data is \emph{algorithmic fairness}
\citep{dwork2012fairness,barocas2016big,corbett2018measure,selbst2019fairness,madaio2020co,mehrabi2021survey,coley2021racial,burhanpurkar2021scaffolding}, where
practical algorithms are developed for pre-processing \citep{feldman2015certifying},
in-processing \citep{zemel2013learning,edwards2015censoring,zafar2017fairness,donini2018empirical,madras2018learning,martinez2020minimax,lahoti2020fairness},
and
post-processing \citep{hardt2016equality,kim2019multiaccuracy,cherepanova2021technical} steps. Nonetheless, there are several challenges when applying fairness algorithms
in practice \citep{beutel2019putting,saha2020measuring,holstein2019improving,madaio2020co}. Specifically, as hinted in Fig.~\ref{fig:teaser}, the fairness generalization guarantee, especially
in over-parameterized models and large datasets, is not well-understood, leading to various practical concerns. We remark that although \citep{kini2021label} claims it is necessary to use multiplicative instead of additive logits adjustments, their motivating example is different from ours and they studied SVM with \emph{fixed and specified budgets} for all inputs. In another closely related work, \citep{cotter2019training}, classical learning theory is being studied without addressing issues raised by imbalanced data.
To the best of our knowledge,
this is the first work tackling the open challenge of fairness generalization with imbalanced data.

\vspace{-0.2cm}
\section{Background}

\paragraph{Notation.} For any $k\in \bN^+$, we use $[k]$ to denote the set $\{1,2,\cdots,k\}$. For a vector $v$, let $v_i$ be the $i$-th coordinate of $v$. We use $\bm{1}$ to denote the indicator function. For a set $S$, we use $|S|$ to denote the cardinality of $S$. For two positive sequences $\{a_k\}$ and $\{b_k\}$, we write $a_k =O(b_k)$ (or $a_n\lesssim b_n$), and $a_k = o(b_k)$, if $\lim_{k\rightarrow\infty}(a_k/b_k) < \infty$ and $\lim_{k\rightarrow\infty}(a_k/b_k) =0$, respectively. We use $\bP$ for probability and $\bE$ for expectation, and we use $\hat \bP$ and $\hat \bE$ for empirical probability and expectation. For the two distributions $\cD_1$ and $\cD_2$, we use $p \cD_1+(1-p)\cD_2$ for $p\in(0,1)$ to denote the mixture distribution such that a sample is drawn with probabilities $p$ and $(1-p)$ from $\cD_1$ and $\cD_2$ respectively. We use $\cN_p(\mu,\Sigma)$ to denote $p$-dimensional Gaussian distribution with mean $\mu$ and variance $\Sigma$.


\paragraph{Fairness notions.} Throughout the paper, we consider datasets consisting of triplets of the form $(x, y, a)$, where $x\in\cX$ is a feature vector, $a\in \cA$ is a sensitive attribute such as race and gender, and $y\in\cY$ is the corresponding label. The underlying random triplets corresponding to $(x,y,a)$ is denoted as $(X,Y,A)$. Our goal is to learn a predictor $h\in\cH:\cX\mapsto \cY$, where $h(X)$ is a prediction of the label $Y$ of input $X$. In this paper, we mainly consider \textit{equalized odds} (EO) \citep{hardt2016equality} that has been widely used in previous literature on fairness. But our method could also be directly used to \textit{equalized opportunity} (EqOpt) given that EqOpt is quite similar to EO. In addition, under certain conditions, our method could also be used to demographic parity (DP), which we will mainly discuss in the Appendix. Specifically, the notions mentioned above are defined as follows.
\begin{itemize}
	\item[(i).] \textit{Equalized odds (EO) and Equalized opportunity (EqOpt)}.   A predictor $h$ satisfies equalized odds if $h(X)$ is conditionally independent of the sensitive attribute $A$ given $Y$: $\bP(h(X)=y|Y=y,A=a)=\bP(h(X)=y|Y=y).$ If $\cY=\{0,1\}$ and we only require $\bP(h(X)=1|Y=1,A=a)=\bP(h(X)=1|Y=1),$ we say $h$ satisfies equalized opportunity.
	\item[(ii).] \textit{Demographic parity (DP)}. A predictor $h$ satisfies demographic parity if $h(X)$ is statistically independent of the sensitive attribute $A$: $\bP(h(X)=Y|A=a)=\bP(h(X)=Y).$
\end{itemize}

\section{Theory-inspired derivation}\label{sec:motivation}
While we will formally introduce our new approach in Section \ref{sec:flexible alg}, this section gives an informal derivation, with an emphasis on insights. We design an 
imbalance-fairness-aware approach that can be flexibly combined with fair learning methods with logits-based losses.

Consider the supervised $k$-class classification problem, where a model $f:\cX\mapsto \bR^k$ provides $k$ scores, and the label is assigned as the class label with the highest score. The corresponding predictor $h(x)=\argmax_i f(x)_i$ if there are no ties. Let us use $\cP_i=\cP(X|Y=i)$ to denote the conditional distribution when the class label is $i$ for $i\in[k]$ and $\cP_{\bal}$ to denote the balanced distribution $\cP_{\text{Idx}}$, where $\text{Idx}$ is uniformly drawn from $[k]$. Similarly, let us use $\cP_{i,s}=\cP(X|Y=i,A=s)$ to denote the conditional distribution when $Y=i$ and $A=s$. The corresponding empirical distributions induced by the training data are $\hat\cP_i$, $\hat\cP_{\bal}$ and $\hat \cP_{i,s}$. For the training dataset $\{(x_j,y_j,a_j)\}_j$, let $S_i=\{j:y_j=i\}$, $S_{i,a}=\{j:y_j=i,a_j=a\}$, and the corresponding sample sizes be $n_i$ and $n_{i,a}$, respectively. Although $\cP_i$, $\cP_{\bal}$ and $ \cP_{i,s}$ are all distributions on $\cX$, we sometimes use notations like $(x,y)\sim \cP_i$ to denote the distribution of $(x,i)$, where $x\sim \cP_i$ . Our goal is to ensure $\cL_{\bal}[f]=\bP_{(x,y)\sim\cP_{\bal}}[f(x)_y<\max_{l\neq y}f(x)_l]$ and the fairness violation error to be as small as possible. 
In order to do that, we need to take the margin of subgroups divided according to sensitive attributes in each label class (so called demographic subgroups in different classes) into account.

\paragraph{Margin trade-off between classes of equalized odds.}
In the setting of standard classification with imbalanced training datasets such as in \citep{cao2019learning,sagawa2020investigation}, the aim is to reach a small balanced test error $\Lbal[f]$. However, in a fair classification setting, our aim is not only to reach a small $\Lbal[f]$, but also to satisfy certain fairness constraints at \textit{test time}.
  Specifically, for EO, the aim is:
 \begin{align*}
  &\min_f \Lbal[f]\\
  \text{s.t.}~\forall y\in\cY, a\in \cA,~~\bP(h(X)&=y|Y=y,A=a)=\bP(h(X)=y|Y=y),
\end{align*}
where we recall that $h(\cdot)=\argmax_i f(\cdot)_i$. We remark here that in addition to the class-balanced loss $\Lbal[f]$, we can also consider the  loss function that is balanced across all demographic subgroups in different classes, the derivation is similar and we omit it here.

Recall our motivating example in Figure \ref{fig:teaser}. Whether the fairness violation error is small \textit{at test time} should also be taken into account. Thus, our \textit{\textbf{performance criterion}} for optimization should be: 
\begin{equation}\label{eq:cri}
 \Lbal[f]+\alpha \Lfair,   
\end{equation}
where $\Lfair$ is a measure of fairness constraints violation that we will specify later, and $\alpha$ is a weight parameter chosen according to how much we care about the fairness constraints violation. 

For simplicity, we start with $\cY=\{0,1\}$ and $\cA=\{a_1,a_2\}$. In the Appendix, we will further discuss the case when  there are \textit{multiple classes and multiple demographic groups}. We also assume the training data is well-separated that all the training samples are perfectly classified and fairness constraints are perfectly satisfied. The setting has been considered in \citep{cao2019learning} and can be satisfied if the model class is rich, for instance, for over-parameterized models such as neural networks. Specifically, if all the training samples are classified perfectly by $h$, not only $\bP_{(x,y)\sim\hat \cP_{\bal}}(h(x)\neq y)=0$ is satisfied, we also have that $\bP_{(x,y)\sim\hat{\cP}_{i,a_j}}(h(x)\neq y)=0$ for all $i\in\cY$ and $a_j\in\cA$. We remark here that $\bP(h(X)=i|Y=i,A=a)=1-\bP_{(x,y)\sim\cP_{i,a}}(h(x)\neq y)$. Denote the margin for class $i$ by $\gamma_i=\min_{j\in S_i} \gamma(x_j,y_j)$, where $\gamma(x,y)=f(x)_y-\max_{l\neq y}f(x)_l$. One natural way to choose $\Lfair$ is to take $\sum_{i\in\cY}|\bP(h(X)=i|Y=i,A=a_1)-\bP(h(X)=i|Y=i,A=a_2)|.$ Then, our performance criterion for optimization in (\ref{eq:cri}) is: 
\begin{equation*}
 \cM[f]=\Lbal[f]+\alpha \sum_{i\in\cY}\left|\bP(h(X)=i|Y=i,A=a_1)-\bP(h(X)=i|Y=i,A=a_2)\right|.
 \end{equation*}
By a similar margin-based generalization arguments in \citep{cao2019learning,kakade2008complexity}, we proved the following theorem.

\begin{theorem}[Informal]\label{thm:eomargin}
With high probability over the randomness of the training data, for $\cY=\{0,1\}$, $\cA=\{a_1,a_2\}$, and for some proper complexity measure of class $\cF$, for any $f\in\cF$,
\begin{align}\label{eq:M}
\begin{split}
\cM[f]\lesssim\sum_{i\in\cY}\frac{1}{\gamma_i}\sqrt{\frac{C(\cF)}{n_i}}+ \sum_{i\in\cY,a\in\cA}\frac{\alpha}{\gamma_{i,a}}\sqrt{\frac{C(\cF)}{n_{i,a}}}
\le \sum_{i\in\cY}\frac{1}{\gamma_i}\sqrt{\frac{C(\cF)}{n_i}}+ \sum_{i\in\cY,a\in\cA}\frac{\alpha}{\gamma_{i}}\sqrt{\frac{C(\cF)}{n_{i,a}}},
\end{split}
\end{align}
where $\gamma_i$ is the margin of the $i$-th class's sample set $S_i$ and $\gamma_{i,a}$ is the margin of demographic subgroup's sample set $S_{i,a}$.
\end{theorem}
Optimizing the upper bound in (\ref{eq:M}) with respect to margins in the sense that 
$g(\gamma_0,\gamma_1)\le g(\gamma_0-\delta,\gamma_1+\delta)$ 
for $g(\gamma_0,\gamma_1)=\sum_{i\in\cY}\frac{1}{\gamma_i \sqrt{n_i}}+\alpha \sum_{i\in\cY,a\in\cA}\frac{1}{\gamma_{i} \sqrt{n_{i,a}}}$ and all $\delta\in[-\gamma_1,\gamma_0]$,
we obtain
$$\gamma_0/\gamma_1=\tilde{n}_0^{-1/4}/\tilde n_1^{-1/4},$$
where the adjusted sample size $$\tilde n_i=\frac{n_i\Pi_{a\in\cA}n_{i,a}}{(\sqrt{\Pi_{a\in\cA}n_{i,a}}+\alpha\sum_{a\in\cA}\sqrt{n_in_{i,a}})^2}$$ for $i\in \{0,1\}$. From Theorem~\ref{thm:eomargin}, we see how sample sizes of each subgroups are taken into account and how they affect the optimal ratio between class margins. Based on this theorem, we will propose our theoretical framework in Section \ref{sec:flexible alg}. A closely related derivation has been used in \citep{cao2019learning}, but their focus is only on the classification error and its generalization. As we will show in Example \ref{ex:example}, when fairness constraints are also considered, their methods could sometimes perform poorly with respect to the generalization of those constraints. We remark here that if we do not consider the fairness constraints violation, then $\alpha=0$, and the effective sample sizes degenerate to $\tilde n_i=n_i$.

For illustration, we demonstrate the advantage of applying our approach to select margins over directly using the margin selection in \citep{cao2019learning} by considering Gaussian models, which is widely used in machine learning theory \citep{schmidt2018adversarially,zhang2021and,deng2021improving}. Specifically, our training data follow distribution: $x\mid y=0 \sim \sum_{i=1}^{2}\pi_{0,a_i} \cN_p(\mu_i, I),~~x\mid y=1 \sim \sum_{i=1}^{2}\pi_{1,a_i} \cN_p(\mu_i+\beta^*, I).$ 
Here, in class $j$, subgroup $a_i$ is drawn with probability $\pi_{j,a_i}$, then, given the sample is from subgroup $a_i$ in class $j$, the data is distributed as a Gaussian random vector. Recall the corresponding training dataset indices of subgroup $a_i$ in class $j$ is denoted as $S_{j,a_i}$, and $|S_{j,a_i}|=n_{j,a_i}$. Consider the case $\alpha=1$ 
, $\pi_{0,a_1}=\pi_{0,a_2}$, and the following class of classifiers: $\mathcal F=\Big\{\bm{1}\{\beta^{*\top} x>c\}:c\in \R \Big\},$ which is a linear classifier class that contains classifiers differ from each other by a translation in a particular direction. 

\begin{example}\label{ex:example}
Given function $f$ and set $S$, let $\dist(f,S)=\min_{x,s\in S}\|f(x)-s\|_2.$ Consider two classifiers $\tilde{f}, f \in \cF$ such that $$\dist(\tilde f, S_0)/\dist(\tilde f, S_1)=\tilde{n}_0^{-1/4}/\tilde n_1^{-1/4}$$ and $\dist(f', S_0)/\dist(f', S_1)=n_0^{-1/4}/ n_1^{-1/4}$. Suppose  $\|\beta^*\|\gg \sqrt{p\log n}$, $\|\mu_i\|<C$, $(\mu_1^*-\mu_2^*)^\top\beta=0$,  
and $\pi_{1,a_2}\le c_1 \pi_{1,a_1}$ for a sufficiently small $c_1>0$, then when $n_0, n_1$ are sufficiently large, with high probability we have $\cM[\tilde f] < \cM[f].$
\end{example}

 \begin{remark}
(1). We provide analyses for the $0$-$1$ loss as our ultimate goal is to strike a balance between good \textbf{test accuracy} and small \textbf{fairness constraints violation}.
If we use surrogates such as the softmax-cross-entropy loss for the $0$-$1$ loss in training, our theoretical analyses still stand
since we always adjust margins based
on the $0$-$1$ loss
as our interests are in quantities such as test accuracy. 
(2). Our analysis can be readily applied to EqOpt and DP constraints under certain conditions.
We provide analyses and experiments for DP in the Appendix.
\end{remark}

\section{Flexible combination with logits-based losses }\label{sec:flexible alg}
\begin{wrapfigure}{r}{0.25\textwidth}
    \vspace{-5pt}
\vspace{-8pt}
        \centering
        \includegraphics[width=\linewidth]{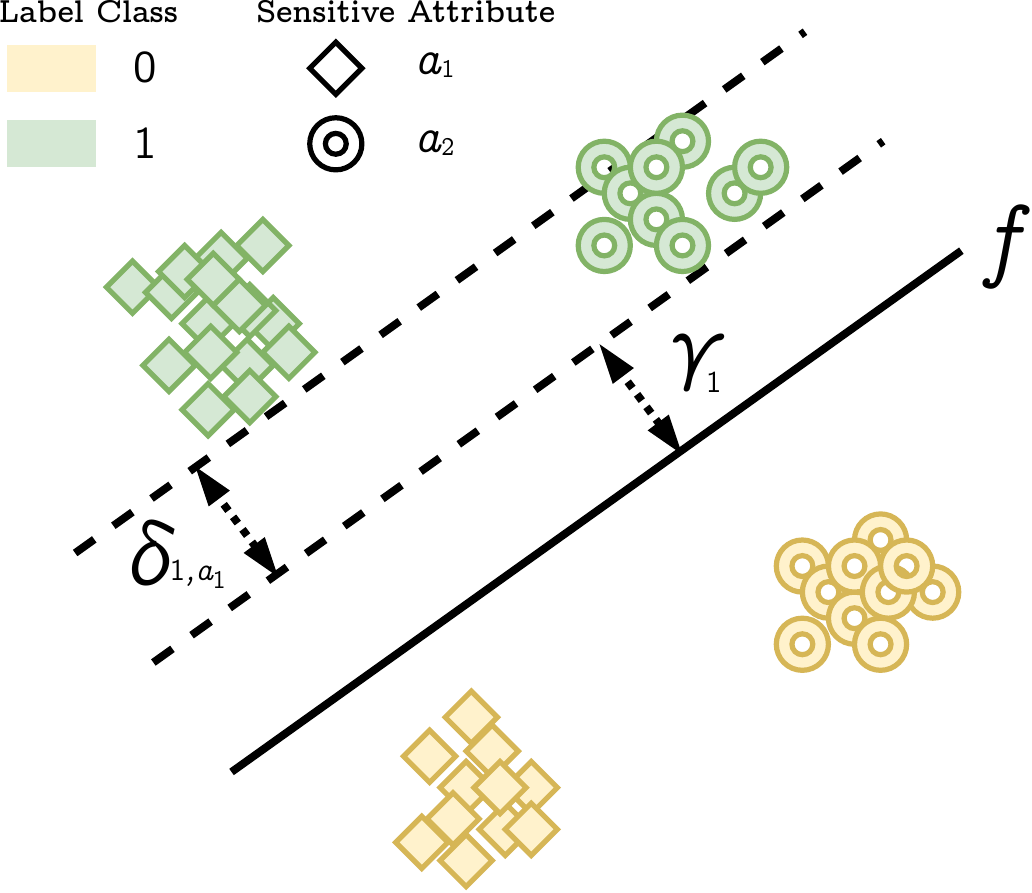}
    	\caption{Illustration of $\delta_{i,a}$
    	and the margin $\gamma$ of classifier:
    	$\delta_{1,a_1}$ is set to be non-negative and
    	$\delta_{1,a_2}$ is set to be
    	zero as the subgroup $(1,a_2)$ is closer to the decision boundary than $(1,a_1)$.
    	}
    	\label{fig:teaser:margin}
    \vspace{-15pt}
\end{wrapfigure}
For ease of exposition, we focus solely on
the EO constraint hereafter and discuss other constraints in the Appendix.
Inspired by the margin trade-off characterized in Section~\ref{sec:motivation}, we propose our FIFA approach for {\bf F}lexible {\bf I}mbalance-{\bf F}airness-{\bf A}ware classification that can be easily combined with different types of logits-based losses, and 
further incorporated into any existing fairness algorithms such as those discussed in Section~\ref{sec:algo}.

Let us recall that in Theorem~\ref{thm:eomargin},
$\gamma_{i,a}=\gamma_{i}+\delta_{i,a}$ and $\delta_{i,a}\ge 0$ (since $\gamma_i=\min\{\gamma_{i,a_1},\gamma_{i,a_2}\}$, also see Fig.~\ref{fig:teaser:margin} for illustration), hence the middle term of Eq.~(\ref{eq:margin})
can be
further upper bounded by the last term in Eq.~(\ref{eq:M}). The final upper bound in Eq.~(\ref{eq:M}) is indeed sufficient for obtaining the margin trade-off between classes. Nonetheless, if we want to further enforce margins for each demographic subgroup in each class, we need to use the refined bound.
Specifically, in Section~\ref{sec:motivation}, we have identified a way to select $\gamma_0/\gamma_1$, based on which we propose to enforce margins for each demographic subgroup's training set $S_{i,a}$ of the form
\begin{equation}\label{eq:margin}
\gamma_{i,a}=\frac{C}{\tilde{n}_i^{1/4}}+\delta_{i,a},
\end{equation}
where $\delta_{i,a}$ and $C$ are all non-negative tuning parameters.
In light of the trade-off between the class margins $\gamma_0/\gamma_1=\tilde{n}_0^{-1/4}/\tilde n_1^{-1/4}$,
we can set $\gamma_i$ of the form $C/\tilde{n}_i^{-1/4}$. 
Given $\gamma_{i,a}\ge \gamma_i$, a natural
choice for margins for subgroups
is Eq.~(\ref{eq:margin}).

\paragraph{How to select $\delta_{i,a}$?}
Knowing the form of margins from the preceding discussions, an outstanding question remains: how to select $\delta_{i,a}$ for imbalanced datasets?
Let $\cY=\{0,1\}$ and $\cA=\{a_1,a_2\}$, within each class $i$, we identify $S_{i,a}$ with the largest cardinality $|S_{i,a}|$ and set the corresponding $\delta_{i,a}=0$.
The remaining $\delta_{i,\cA\backslash a}$ are tuned as a non-negative  parameter. As a further illustration, without loss of generality, assume for all $i$, $|S_{i,a_1}|\ge|S_{i,a_2}|$.
Thus selected $\{\delta_{i,a}\}_{i,a}$ ensures the upper bound in the middle of Eq.~(\ref{eq:M}) is tighter in the sense that for any $\delta>0$,
\begin{equation*}
\sum_{i\in\cY}\Big(\frac{1}{\gamma_{i} \sqrt{n_{i,a_1}}}+\frac{1}{(\gamma_{i}+\delta) \sqrt{n_{i,a_2}}}\Big)\le  \sum_{i\in\cY}\Big(\frac{1}{(\gamma_{i}+\delta) \sqrt{n_{i,a_1}}}+\frac{1}{\gamma_{i} \sqrt{n_{i,a_2}}}\Big).
\end{equation*}

In the Appendix, we will present how to choose $\delta_{i,a}$'s when there are multiple demographic groups.
Briefly speaking, our results similar to the above inequality are proved
by an application of the rearrangement inequality. Simple as it is, the
high-level view is meaningful -- 
the decision boundaries of a fair predictor should be farther away from the less-frequent subgroup than the more-frequent subgroup to ensure better fairness generalization.

\paragraph{Flexible imbalance-fairness-aware (FIFA) approach.} We will demonstrate how to apply the above motivations to design better margin losses. Loosely speaking, we consider a logits-based loss
\begin{equation*}
\ell((x,y);f)=\ell(f(x)_y, \{f(x)_i\}_{i\in\cY\backslash y}),
\end{equation*}
which is non-increasing with respect to its first coordinate if we fix the second coordinate. Such losses include (i). $0$-$1$ loss: $\bm{1}\{f(x)_y<\max_{i\in\cY\backslash y}f(x)_{i}\}$. (ii). Hinge loss: $\max\{\max_{i\in\cY\backslash y}f(x)_{i}-f(x)_y,0\}$. (iii). Softmax-cross-entropy loss: $-\log e^{f(x)_y}/(e^{f(x)_y}+\sum_{i\neq y}e^{f(x)_i})$.

Our flexible imbalance-fairness-aware (FIFA) approach modifies the above losses by enforcing margin of the form in Eq.~(\ref{eq:margin}). Specifically, we use the following loss function \textbf{\textit{during training}}
\begin{equation} \label{eq:fifaloss}
\lfifa((x,y,a);f)=\ell(f(x)_y-\Delta_{y,a}, \{f(x)_i\}_{i\in\cY\backslash y})
\end{equation}
where $\Delta_{i,a}=C/\tilde{n}_i^{1/4}+\delta_{i,a}$. We remark here $\ell_{\text{FIFA}}((x,y,a);f)$ is used only during training phase, where we allow access to sensitive attribute $a$. In the test time, we only need to use $f$ but \textit{\textbf{not}} $a$.

\section{Example: combining FIFA with reductions-based fair algorithms}\label{sec:algo}

In this section, we demonstrate the power of our approach by combining it with a popular reduction-based fair classification algorithm \citep{agarwal2018reductions} \textit{as an example}. In Section~\ref{sec:exp}, we show that incorporating our approach can bring a significant gain in terms of both combined loss and fairness generalization comparing with directly applying their method in vanilla models trained with softmax-cross-entropy losses. 
The reduction approach proposed in \citep{agarwal2018reductions} has two versions: (i).~{\bf Exponentiated gradient} ({\bf ExpGrad}) that produces a randomized classifier; and (ii).~{\bf Grid search} ({\bf GridS}) that produces a deterministic classifier. Our approach can be combined with both.

\paragraph{Exponentiated gradient (ExpGrad).}
We first briefly describe the algorithm here. For $\cY=\{0,1\}$, by \citep{agarwal2018reductions}, EO constraints could be rewritten as $M\mu(h)\le c$ for certain $M$ and $c$, where $\mu_j(h)=\bE[h(X)|E_j]$ for $j\in\mathcal{J}$, $M\in\bR^{|\cK|\times |\cJ|}$, and $c\in\bR^{\cK}$. Here, $\cK=\cA\times\cY\times\{+,-\}$ ($+,-$ impose positive/negative sign so as to recover $|\cdot|$ in constraints) and $\cJ=(\cA\cup \{*\})\times\{0,1\}$. $E_{(a,y)}= \{A=a,Y=y\}$ and $E_{(*,y)}=\{Y=y\}$. Let $\error(h)=\bP(h(X)\neq Y)$, instead of considering $\min_{h\in\cH} \error(h)$ such that $M\mu(h)\le c$,
{\bf ExpGrad} obtains the best \emph{randomized classifier}, by sampling a classifier $h\in\cH$ from a distribution over $\cH$.
Formally, this optimization can be formulated as $$\min_{Q\in\Delta_{\cH}} \error(Q)~~\text{such that}~~ M\mu(Q)\le c,$$ where $$\error(Q)=\sum_{h\in\cH}Q(h)\error(h),$$ and $\mu(Q)=\sum_{h\in\cH}Q(h)\mu(h)$, $Q$ is a distribution over $\cH$, and $\Delta_\cH$ is the collection of distributions on $\cH$. Let us further use $\err(Q)$ and $\hat{\mu}(Q)$ to denote the empirical versions and also allows relaxation on constraints by using $\hat c=c+\epsilon$, where $\hat c_k=c_k+\epsilon_k$ for relaxation $\varepsilon_k\ge 0$. By classic optimization theory, it could be transferred to a saddle point problem, and \citep{agarwal2018reductions} aims to solve the following prime dual problems simultaneously for $L(Q,\lambda)=\err(Q)+\lambda^\top(M\hat\mu(Q)-\hat c)$:
\begin{align*}
(\textbf{P}):~\min_{Q\in\Delta} \max_{\lambda\in\bR^{|\cK|}_+,\|\lambda\|_1\le B} L(Q,\lambda),\qquad (\textbf{D}):~\max_{\lambda\in\bR^{|\cK|}_+,\|\lambda\|_1\le B}\min_{Q\in\Delta} L(Q,\lambda).
\end{align*}
To summarize, $\textbf{ExpGrad}$ takes training data $\{(x_i,y_i,a_i)\}_{i=1}^n$, function class $\cH$, constraint parameters $M,\hat c$, bound $B$, accuracy tolerance $v>0$, and learning rate $\eta$ as inputs and outputs $(\hat Q,\hat\lambda)$, such that $L(\hat Q,\hat\lambda)\le L(Q,\hat\lambda)+\nu$ for all $Q\in \Delta_\cH$ and $L(\hat Q,\hat\lambda)\le L(\hat Q,\lambda)-\nu$ for all $\lambda\in\bR^{|\cK|}_+,\|\lambda\|_1\le B$, and $(\hat Q,\hat\lambda)$ is called a $\nu$-approximate saddle point. As implemented in \citep{agarwal2018reductions},
$\cH$ roughly consists of $h(x)=\bm{1}\{f(x)_1\ge  f(x)_{0}\}$ for $f\in\cF$ 
(in fact, a smoothed version is considered in \citep{agarwal2018reductions})
and gives
\begin{equation*}
\err(Q)=\sum_{h\in\cH}\hat{\bP}(h(X)\neq Y)Q(h)=\hat{\bP}(f(X)_Y<f(X)_{\{0,1\}\backslash Y})Q(h).
\end{equation*}
To combine our approach, we consider optimizing $$\err^{\text{new}}(Q)=\sum_{h\in\cH}\hat{\bP}(f(X)_Y-\Delta_{Y,A}\le f(X)_{\{0,1\}\backslash Y})Q(h),$$
such that $M\hat{\mu}^{\text{new}}(Q)\le \hat c$, where $\hat{\mu}^{\text{new}}(Q)=\sum_{h\in\cH}Q(h)\hat{\mu}^{\text{new}}(f)$ and $\hat{\mu}^{\text{new}}_j(f)=\hat{\bP}(f(X)_Y-\Delta_{Y,A}> f(X)_{\{0,1\}\backslash Y}|E_j)$. We can modify \textbf{ExpGrad} to optimize prime dual problems simultaneously for $$L^{\text{new}}(Q,\lambda)=\err^{\text{new}}(Q)+\lambda^\top(M\hat{\mu}^{\text{new}}(Q)-\hat c).$$ 
In practice,
while Section~\ref{sec:motivation} is 
motivated for deterministic classifiers, FIFA works for the randomized version too -- the modified \textbf{ExpGrad} can be viewed as encouraging a distribution $Q$ that puts more weights on classifiers with a certain type of margin trade-off between classes. Moreover, the modified algorithm enjoys similar convergence guarantee as the original one. 
\renewcommand{\algorithmicrequire}{\textbf{Input:}}
\renewcommand{\algorithmicensure}{\textbf{Output:}}
\newcommand{\LineComment}[1]{$\triangleright$ \textit{#1}}

\begin{wrapfigure}{L}{0.5\textwidth}
\vspace{-25pt}
\begin{minipage}{0.5\textwidth}
  \begin{algorithm}[H]
  \footnotesize
    \caption{FIFA Combined Grid Search}
    \begin{algorithmic}[1]
        \REQUIRE fairness algorithm \textbf{GridS}, training data set
        $\{x_i, y_i, a_i\}_{i=1}^n$,
        fairness tolerance $\eps$, margins $\{\Delta_{y,a}\}_{y,a}$,
        a classifier $h(\cdot;\theta)$.
      \ENSURE the learnt classifier $h^*$.
      \STATE Load training data to \textbf{GridS}. 
      \STATE \textbf{GridS} produces a set of reduction-labels $\hat{y}_{\text{train}}$ and a set of sample weights $w_{\text{train}}$ based on the type
      of fairness constraint and tolerance $\epsilon$.
      \FOR{$t=1,2,\ldots, T$}
        \STATE Compute the FIFA loss $\lfifa$
        via (\ref{eq:fifaloss}) using reduction-labels $\hat{y}_{\text{train}}$ (in mini-batches).
        \STATE Update $\theta$ in $h$ using back-propagation.
        \STATE Logging training metrics using true labels
        $\{y_{i}\}_{i=1}^n$ and attributes $\{a_{i}\}_{i=1}^n$.
      \ENDFOR
    \end{algorithmic} \label{algo:fifa}
  \end{algorithm} 
\end{minipage}
\vspace{-20pt}
\end{wrapfigure}
\begin{theorem} \label{thm:convergence}
Let $\rho=\max_f\|M\hat{\mu}^{\text{new}}(f)-\hat c\|_\infty$. For $\eta=\nu/(2\rho^2B)$, the modified \textbf{ExpGrad} will return a $\nu$-approximate saddle point of $L^{\text{new}}$ in at most $4\rho^2B^2\log(|\cK|+1)/\nu^2$ iterations.
\end{theorem}

\paragraph{Grid search (GridS).} 
When the number of constraints is small, e.g., when there are only few sensitive attributes, 
one may directly perform a grid search on the $\lambda$ vectors to identify the
\emph{deterministic classifier}
that attains the 
best trade-off between accuracy and fairness.
In practice, 
\textbf{GridS} is preferred for larger models due to its memory efficiency, since \textbf{ExpGrad} needs to store all intermediate models to
compute the randomized classifier at
prediction time, which is less feasible
for over-parameterized models.
We describe our flexible approach in Algorithm~\ref{algo:fifa} that combines with 
\textbf{GridS} used in practice in the official code base \texttt{FairLearn} (\citep{bird2020fairlearn}). 


\begin{remark}
As we stated, the above algorithm is just one of the examples that can be combined with our approach. FIFA can also be applied to many other popular algorithms such as fair representation \citep{madras2018learning} and decision boundary approach \citep{zafar2019fairness}. We will discuss them in more detail in the Appendix.
\end{remark}


\section{Experiments}\label{sec:exp}

We now use our flexible approach on
several datasets in the classification
task with a sensitive attribute. Although our method is proposed for over-parameterized models, it can also boost the performance on small models. Depending on the specific
dataset and model architectures, we use
either the grid search or the exponentiated gradient method developed by \citep{agarwal2018reductions} as fairness algorithms to enforce the
fairness constraints, while adding our FIFA loss in the inner training loop. Note that our method can be combined with other fairness algorithms.

\paragraph{Datasets.} We choose both a large image dataset and two simpler datasets. We use the official train-test split of these datasets.
More details and statistics are in the Appendix. \textbf{(i).~CelebA} (\citep{liu2015faceattributes}): the task is to predict whether
    the person in the image has blond hair or not where the sensitive attribute is the gender of the person.
\textbf{(ii).~AdultIncome} (\citep{Dua2019uci}): the task is to predict whether the income is above $50$K per year, where the sensitive attribute is the gender.
\textbf{(iii).~DutchConsensus} (\citep{dutch2001data}): the task is predict whether an individual has a prestigious occupation and the sensitive attribute is the gender.
Both AdultIncome and DutchConsensus datasets are also used in \citep{agarwal2018reductions}.

\begin{wrapfigure}{r}{0.3\textwidth}
    \vspace{-12pt}
        \centering
        \includegraphics[width=\linewidth]{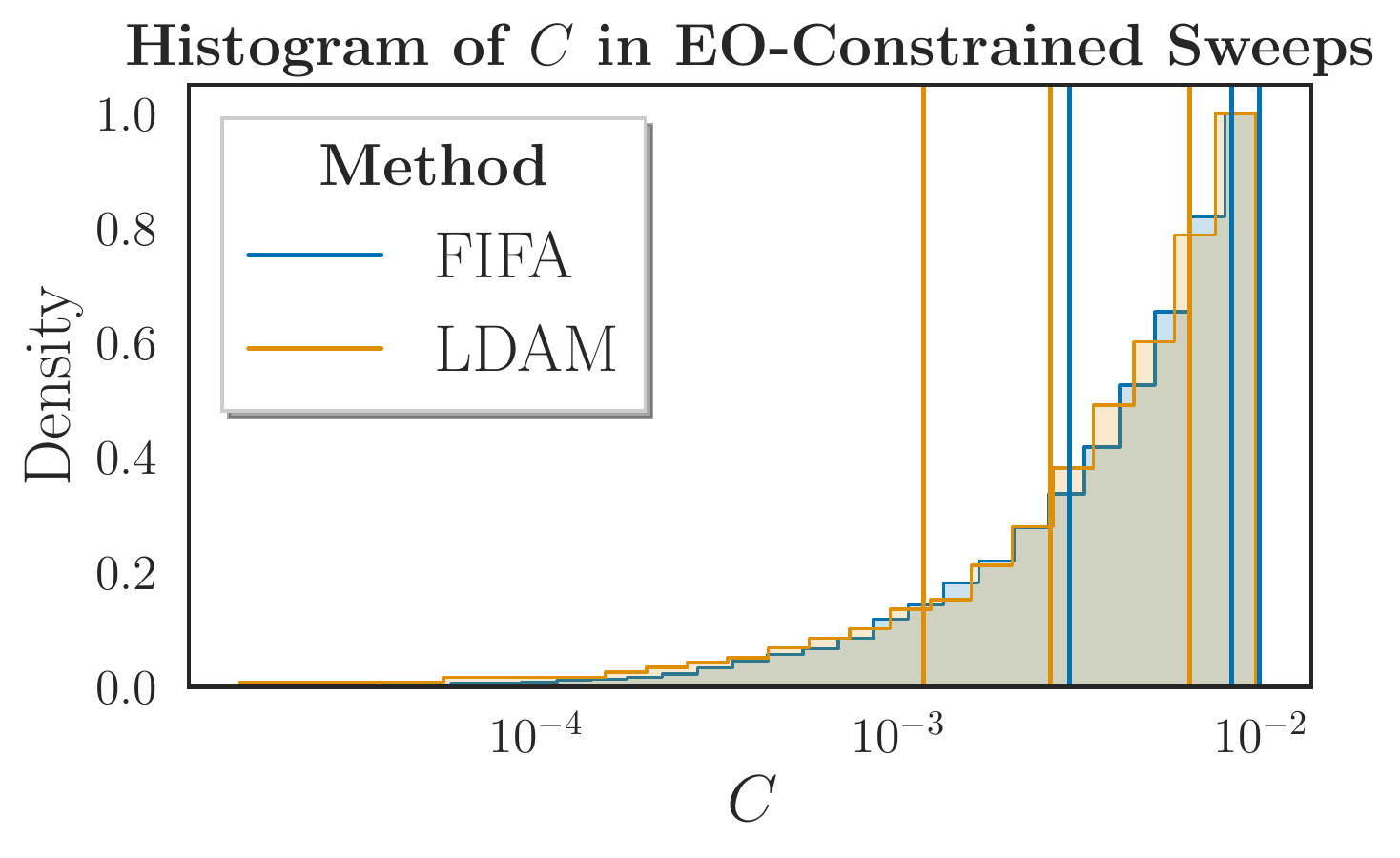}
    	\caption{Histogram (cumulative density) of hyper-parameter $C$ in the sweeps for FIFA and LDAM. Vertical lines mark the values corresponding to the best performing models in Table~\ref{tab:celeba}.} \label{fig:hist:C}
    \vspace{-10pt}
\end{wrapfigure}
\captionsetup{font=normal,labelfont={bf},skip=10pt}
\paragraph{Method.} Due to computational feasibility (\textbf{ExpGrad} needs to store all intermediate models at
prediction time), we
combine Grid Search with FIFA for the CelebA dataset and ResNet-18 and use both
Grid Search and Exponentiated Gradient on the AdultIncome with logistic regression. Besides $C$ and $\delta_{i,a}$, we also treat $\alpha$ as tuning parameters (in Eq.~\eqref{eq:M}). We then perform hyper-parameter
sweeps on the grids (if used) over $C$, $\delta_{i,a}$ and $\alpha$ for FIFA, and grids (if used) for vanilla training (combine fairness algorithms with the vanilla softmax-cross-entropy loss).
The sweeps are done on the \texttt{wandb} platform \citep{wandb}, where all hyper-parameters except for the grid, are searched using its built-in Bayesian backend.
All models for the same dataset are trained with a fixed number of epochs where the training accuracies converge. Batch training with size $128$ is used for CelebA and full batch training is used for AdultIncome. More details are included in the Appendix. We compare FIFA with directly applying fair algorithms to NN's trained with softmax-cross-entropy loss, which is the most natural baseline. Also as a special case of FIFA, when $\delta_{i,a} = 0$ for all $i, a$ and $\alpha = 0$ the FIFA loss degenerates to the LDAM loss proposed in \citep{cao2019learning} that is not as fairness-aware as FIFA, so we briefly discuss it in Table \ref{tab:celeba} too. FIFA further finetunes $\delta_{i,a}$ and $\alpha$, and to ensure a fair comparison, we set the same coverage for the the common hyper-parameter $C$ in the sweeps, as shown in Fig.~\ref{fig:hist:C}. 

\paragraph{Evaluation and Generalization.} When evaluating the model, we are mostly interested in the generalization performance measured by a \emph{combined loss} that take into consideration both fairness violation and balanced error. We define the combined loss as $\Lcomb[f] = \frac{1}{2}\Lbal[f] + \frac{1}{2}\Lfair[f]$, 
which favors those classifiers that have a equally well-performance in terms
of classification and fairness.
We consider both the value of the combined loss evaluated on the test set $\cS_{\text{test}}$, and the \emph{generalization gap} for a loss $\cL$ is
defined as $\gap[\cL, f] = \abs{\cL[f](\cS_{\text{test}})- \cL[f](\cS_{\text{train}})}.$
\subsection{Effectiveness of FIFA on over-parameterized models} \label{sec:exp:large}
\definecolor{Gray}{gray}{0.85}
\newcolumntype{g}{>{\columncolor{Gray}}c}
\begin{table*}
    \scriptsize\centering
    \resizebox{\textwidth}{!}{%
\begin{tabular}{lllllllllll}
\toprule
               \multicolumn{2}{c}{\bf Fairness Tolerance $\epsilon$} & \multicolumn{3}{c}{\bf 0.01} & \multicolumn{3}{c}{\bf 0.05} & \multicolumn{3}{c}{\bf 0.10} \\
               \multicolumn{2}{c}{\bf Method}  & {\bf FIFA} & {\bf LDAM} & {\bf Vanilla} & {\bf FIFA} & {\bf LDAM} & {\bf Vanilla} & {\bf FIFA} & {\bf LDAM} & {\bf Vanilla} \\
\midrule
\multirow{3}{*}{\bf Combined Loss} & Train &   7.37\% &   5.22\% &      7.14\% &                 5.46\% &   5.47\% &      8.84\% &                 5.92\% &   8.48\% &      8.90\% \\
               & Test &   {\bf 6.71\%} &   7.29\% &     14.01\% &                 {\bf 6.34\%} &   7.38\% &     13.05\% &                 {\bf 6.54\%} &   7.34\% &     16.71\% \\
               & Gap &   {\bf 0.66\%} &   2.07\% &      6.87\% &                 {\bf 0.88\%} &   1.91\% &      4.21\% &                 {\bf 0.62\%} &   1.14\% &      7.82\% \\
\midrule
\multirow{3}{*}{\bf Fairness Violation} & Train &   5.31\% &   2.32\% &      6.69\% &                 2.63\% &   2.57\% &      9.45\% &                 3.11\% &   6.93\% &     11.37\% \\
               & Test &   {\bf 2.75\%} &   5.39\% &     20.29\% &                 {\bf 3.29\%} &   5.57\% &     17.92\% &                 {\bf 2.65\%} &   2.96\% &     26.15\% \\
               & Gap &    {\bf 2.57\%} &   3.07\% &     13.59\% &                 {\bf 0.66\%} &   3.00\% &      8.47\% &                 {\bf 0.46\%} &   3.97\% &     14.78\% \\
\bottomrule
\end{tabular}
    }
\caption{Grid search with EO constraint on CelebA 
    dataset \citep{liu2015faceattributes} using ResNet-18, best results with respect to test combined loss  among sweeps of hyper-parameters are shown.
    As an interesting special case of our FIFA method, we note that although the LDAM method improves the performance compared with vanilla GS, it is not as effective as our method.}
    \label{tab:celeba}
\end{table*}

\begin{figure*}[t]
    \centering
    \begin{subfigure}[t]{0.25\textwidth}
	    \centering
	    \includegraphics[width=\linewidth]{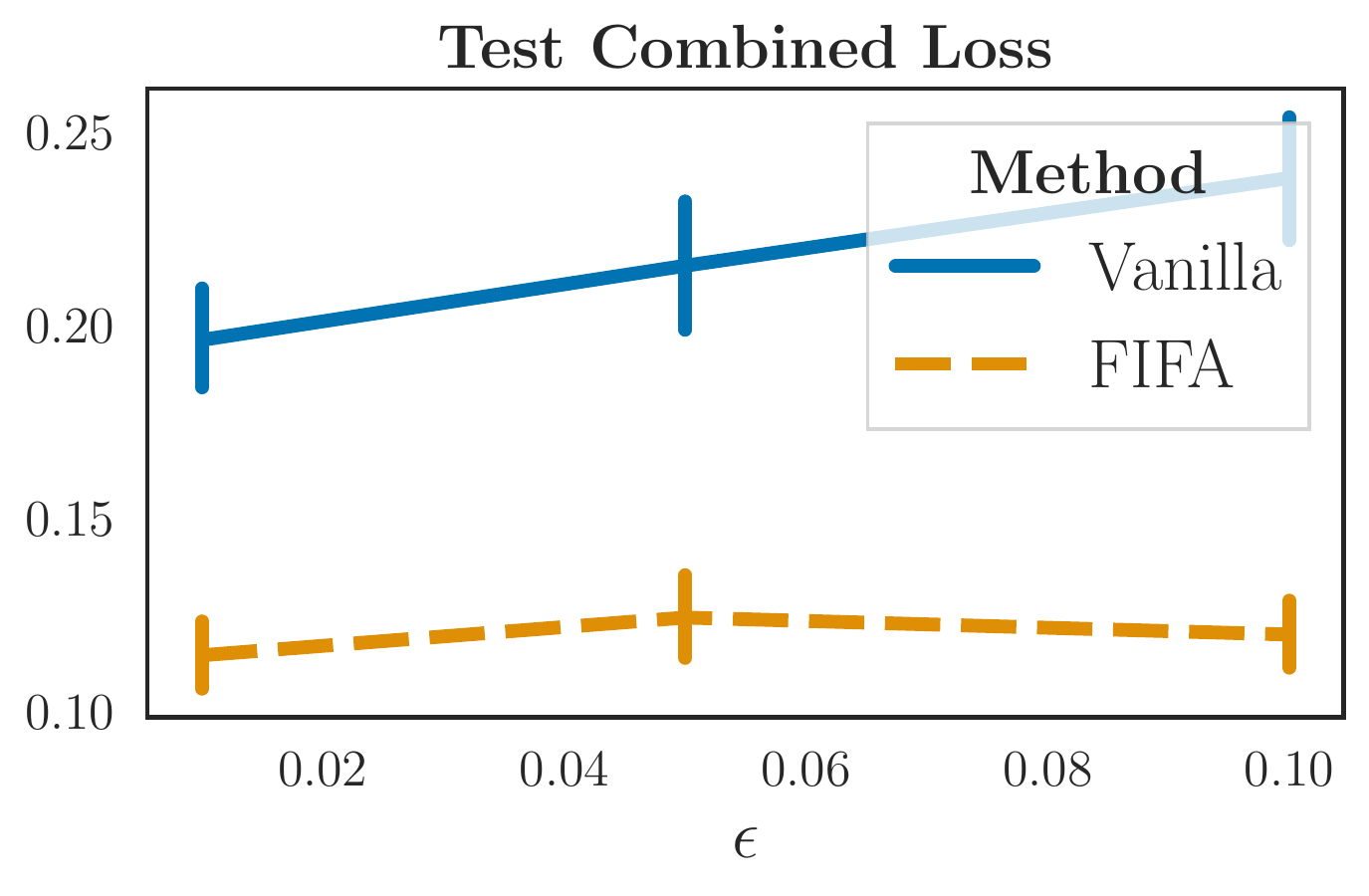}
	    \caption{$\Lcomb[f](\cS_{\text{test}})$.} \label{fig:celeba:gen:testcomb}
	\end{subfigure}%
	\begin{subfigure}[t]{0.25\textwidth}
	    \centering
	    \includegraphics[width=\linewidth]{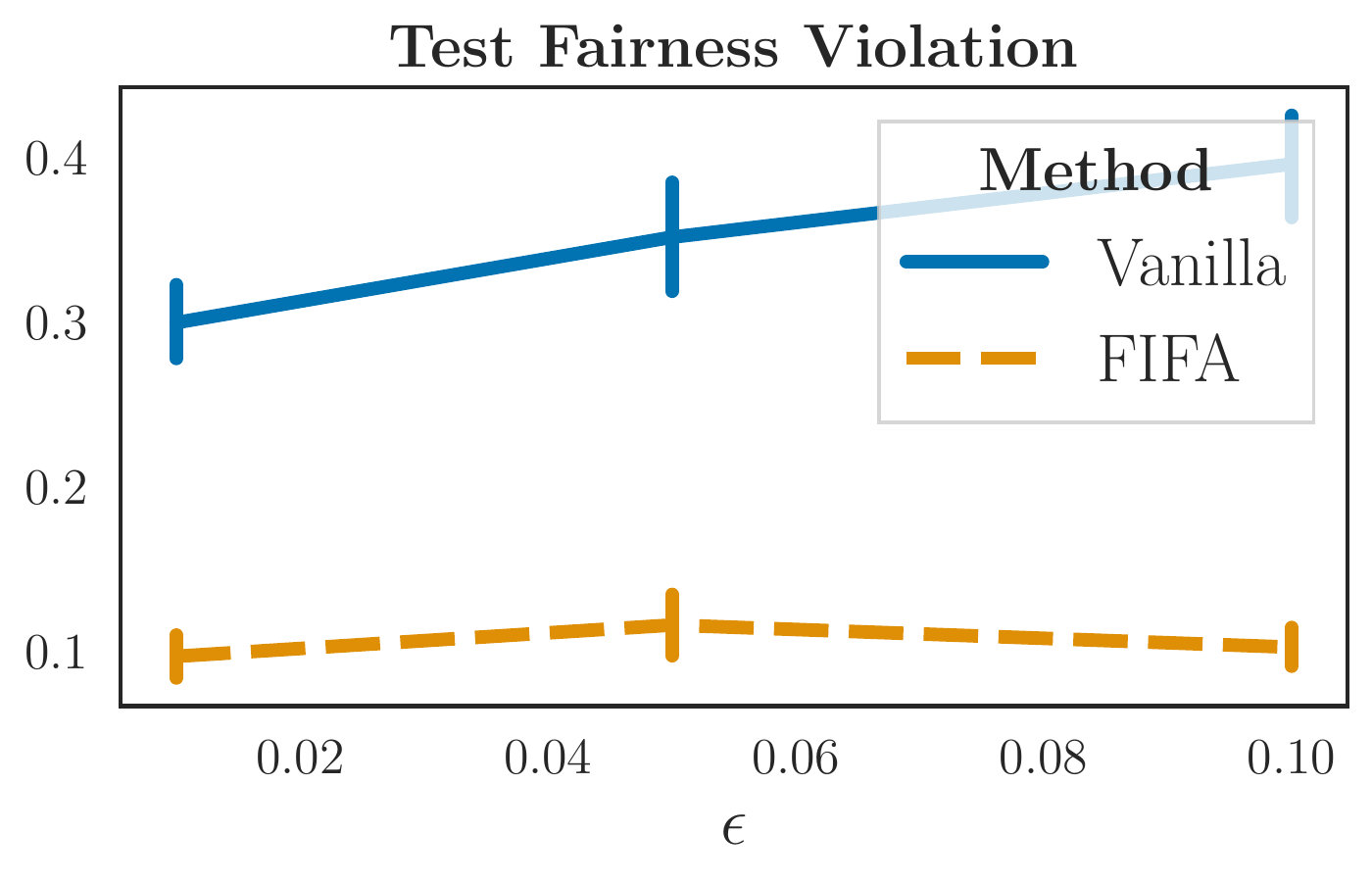}
	    \caption{$\Lfair[f](\cS_{\text{test}})$.} \label{fig:celeba:gen:testfair}
	\end{subfigure}%
    \begin{subfigure}[t]{0.25\textwidth}
	    \centering
	    \includegraphics[width=\linewidth]{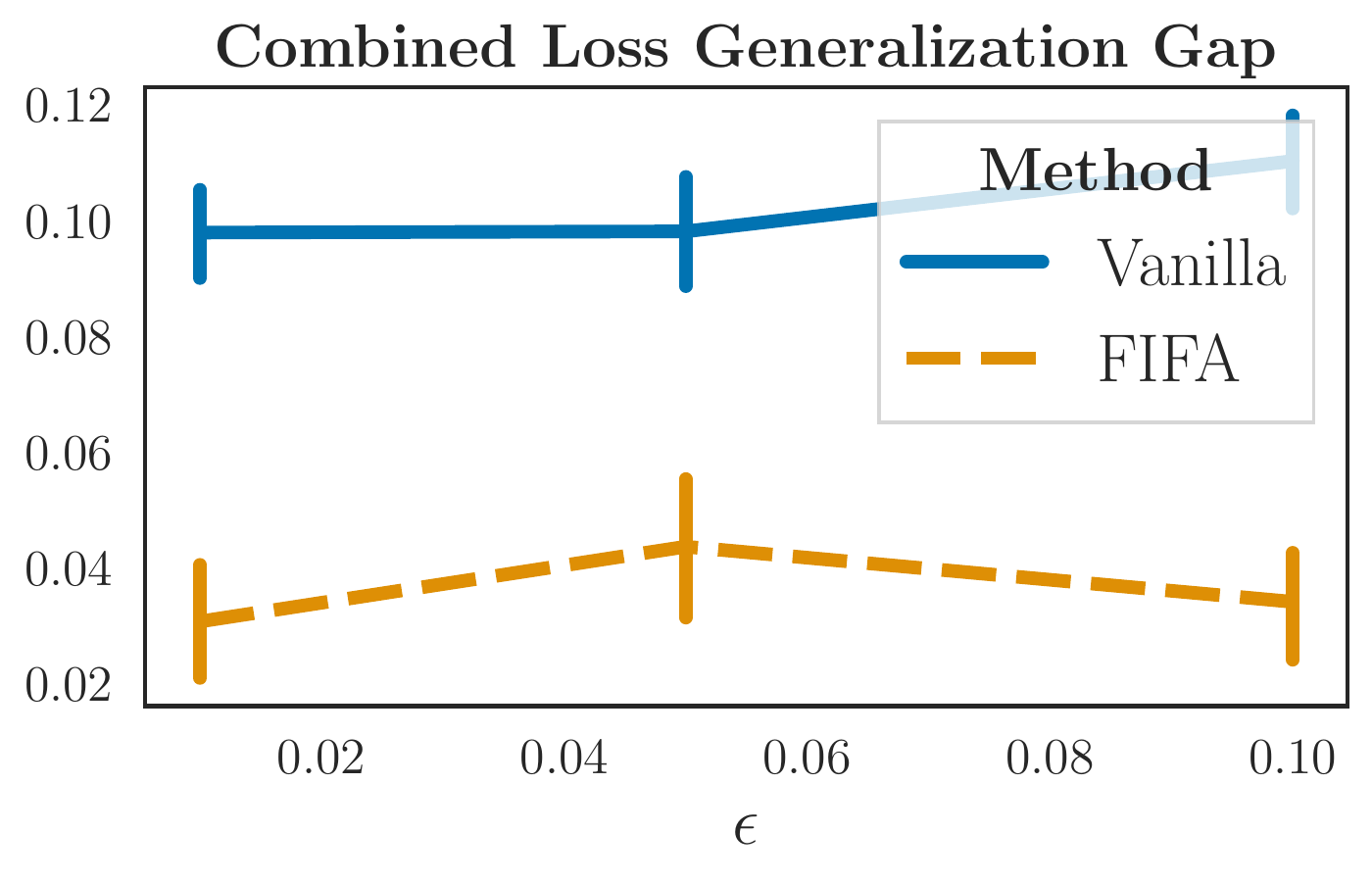}
	    \caption{$\widehat{\gap}[\Lcomb, f]$.} \label{fig:celeba:gen:combgap}
	\end{subfigure}%
	\begin{subfigure}[t]{0.25\textwidth}
	    \centering
	    \includegraphics[width=\linewidth]{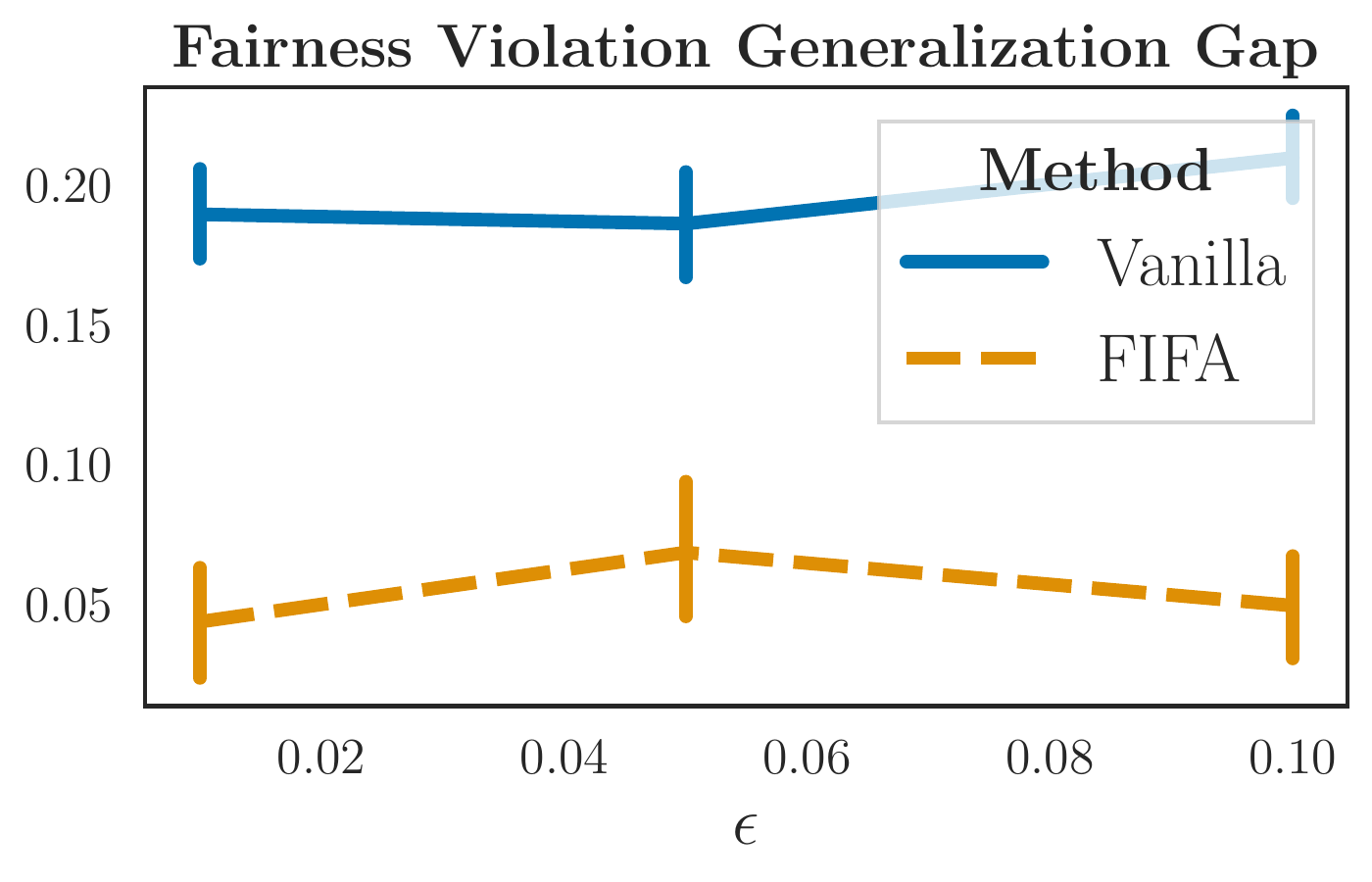}
	    \caption{$\widehat{\gap}[\Lfair, f]$.} \label{fig:celeba:gen:fairgap}
	\end{subfigure}%
	\caption{Loss on test set (\ref{fig:celeba:gen:testcomb},\ref{fig:celeba:gen:testfair}) and generalization gaps (\ref{fig:celeba:gen:combgap},\ref{fig:celeba:gen:fairgap})
	of the combined loss and the fairness loss on CelebA dataset. We repeat the experiment for $20$ times using the hyper-parameters corresponding to the best-performing models in Table~\ref{tab:celeba}. Solid blue line marks the grid search combined with vanilla training whereas dashed orange line marks grid search combined with the FIFA loss. We also plot $95\%$ confidence band based on $20$ repeated runs.
	We observe that our method FIFA has significantly better generalization performance in terms of both smaller losses on the test set as well as narrower generalization gaps.
	}
	\label{fig:celeba:gen}
\end{figure*}

In this subsection, we thoroughly analyze the results from
applying FIFA to over-parameterized models for the CelebA dataset using ResNet-18. We use the grid search algorithm
with fairness violation tolerance parameter $\eps\in\{0.01,0.05,0.1\}$ (with a little abuse of notation) for all constraints.
We perform sweeps on hyper-parameters
$C \in [0,0.01]$, $\alpha \in [0,0.01]$, $\delta_{0, \text{Male}}, \delta_{1, \text{Male}} \in [0,0.01]$,
and $\delta_{0, \text{Female}}=\delta_{1, \text{Female}}=0$.
As a special case that may be of interest, when $\alpha=0$ and 
$\delta_{0, \text{Male}}=\delta_{1, \text{Male}} =0$,
the FIFA loss coincides with the LDAM loss proposed in
\citep{Dua2019uci}, with one common hyper-parameter $C\in[0, 0.01]$.
We log the losses on the whole training and test set.
We summarize our main findings below and give more details in the Appendix including experiments with DP constraints, delayed-reweighting (DRW, \citep{cao2019learning}), and reweighting methods.

\paragraph{Logits-based methods improve fairness generalization.}
We summarize the best results for each method
under different tolerance parameter $\eps$
in Table~\ref{tab:celeba}. We note that
both FIFA and LDAM \textit{\textbf{significantly}} improve
the test performance of both
combined loss and fairness violation
among all three choices of $\eps$, while
having comparable training performance
This implies the effectiveness and necessity of using
logits-based methods to ensure a better fairness generalization.

\paragraph{FIFA accommodates for both fairness generalization and dataset imbalance.}
Although both logits-based method improve generalization as seen in Table~\ref{tab:celeba},
our method FIFA has significantly better generalization performance
compared with LDAM, especially in terms of fairness violation.
For example, when $\eps=0.01$ and $0.05$, FIFA achieves
a test fairness violation that is at least $2\%$
smaller compared with LDAM. This further demonstrates the
importance of our theoretical motivations.

\paragraph{Improvements of generalization are two-fold for 
FIFA.}
When it comes to generalization, two relevant notions
are often used, namely the absolute performance on the test set, and also the generalization gap between the training and test set. We compute the generalization gap in Table~\ref{tab:celeba} for both combined loss and fairness violation. We observe that FIFA generally dominates LDAM and vanilla in terms of both test performance and generalization gap.
We further illustrate this behavior in Fig.~\ref{fig:celeba:gen}, where we give $95\%$
confidence band over randomness in training. We note that
our FIFA significantly outperforms vanilla in a large margin in terms of both generalization notions, and the improvements are mostly
due to better fairness generalization.
In fact, as suggested by the similarity in the shapes of curves between Fig.~\ref{fig:celeba:gen:combgap} and
Fig.~\ref{fig:celeba:gen:fairgap},
fairness generalization dominates 
classification generalization, and thus improvements in fairness generalization elicit
more prominently overall.

\paragraph{Towards a more efficient Pareto frontier.}
Let us recall our motivating example in Fig.~\ref{fig:teaser} where a sufficiently well-trained ResNet-10 demonstrate stellar classification and fairness performance on the training set but poor fairness generalization on the test set. In Fig.~\ref{fig:celeba:pareto} we plot the Pareto frontier of balanced classification error ($\Lbal$) and
fairness violation ($\Lfair$) for all three choices of $\eps$. We observe that
the models trained by combining grid search with the FIFA loss achieve frontiers that are lower and more centered
compared with those trained in vanilla losses with grid search. This reaffirms that our FIFA approach mitigates the fairness generalization problem decently well.

\begin{table*}
    \scriptsize\centering
    \resizebox{\textwidth}{!}{%
\begin{tabular}{llllllllll}
\toprule
                      \multicolumn{2}{c}{\bf Dataset} & \multicolumn{4}{c}{\bf AdultIncome} & \multicolumn{4}{c}{\bf DutchConsensus} \\
                      \multicolumn{2}{c}{\bf Metric}         & \multicolumn{2}{c}{\bf Combined Loss} & \multicolumn{2}{c}{\bf Fairness Violation} & \multicolumn{2}{c}{\bf Combined Loss} & \multicolumn{2}{c}{\bf Fairness Violation} \\
                      $\epsilon$ & {\bf Method}         &         Train &   Test &              Train &   Test &          Train &   Test &              Train &  Test \\
\midrule
\multirow{2}{*}{0.01} & FIFA &      15.7217\% & {\bf 13.4812\%} &            7.8863\% &  {\bf 2.7776\%} &       12.8013\% & {\bf 13.1686\%} &            3.7220\% & {\bf 4.3532\% }\\
                      & Vanilla &      13.9561\% & 14.3001\% &            6.7861\% &  6.7475\% &       12.8444\% & 13.2267\% &            3.7935\% & 4.4323\% \\
\multirow{2}{*}{0.05} & FIFA &      13.5634\% & {\bf 13.5491\%} &            5.7088\% &  {\bf 4.9315\%} &       12.8820\% & {\bf 13.2228\%} &            3.8525\% & {\bf 4.4323\%} \\
                      & Vanilla &      14.4697\% & 14.8647\% &            7.5962\% &  7.8575\% &       12.8818\% & 13.2236\% &            3.8525\% & 4.4323\% \\
\multirow{2}{*}{0.10} & FIFA &      13.5857\% & {\bf 13.9043\%} &            6.1217\% &  {\bf 5.9689\%} &       12.8717\% & {\bf 13.1748\%} &            3.8326\% & {\bf 4.3532\%} \\
                      & Vanilla &      15.5342\% & 15.9387\% &            9.7514\% & 10.0750\% &       12.8757\% & 13.2099\% &            3.8392\% & 4.4059\% \\
\bottomrule
\end{tabular}
    }
\caption{Exponentiated gradient with EO constraint on the AdultIncome and DutchConsensus 
    datasets using logistc regression (as a one-layer neural net), best results with respect to test combined loss among sweeps of hyper-parameters are shown.}
    \label{tab:adultincome}
\end{table*}

\begin{figure*}[t]
    \centering
    \begin{subfigure}[t]{0.33\textwidth}
	    \centering
	    \includegraphics[width=\linewidth]{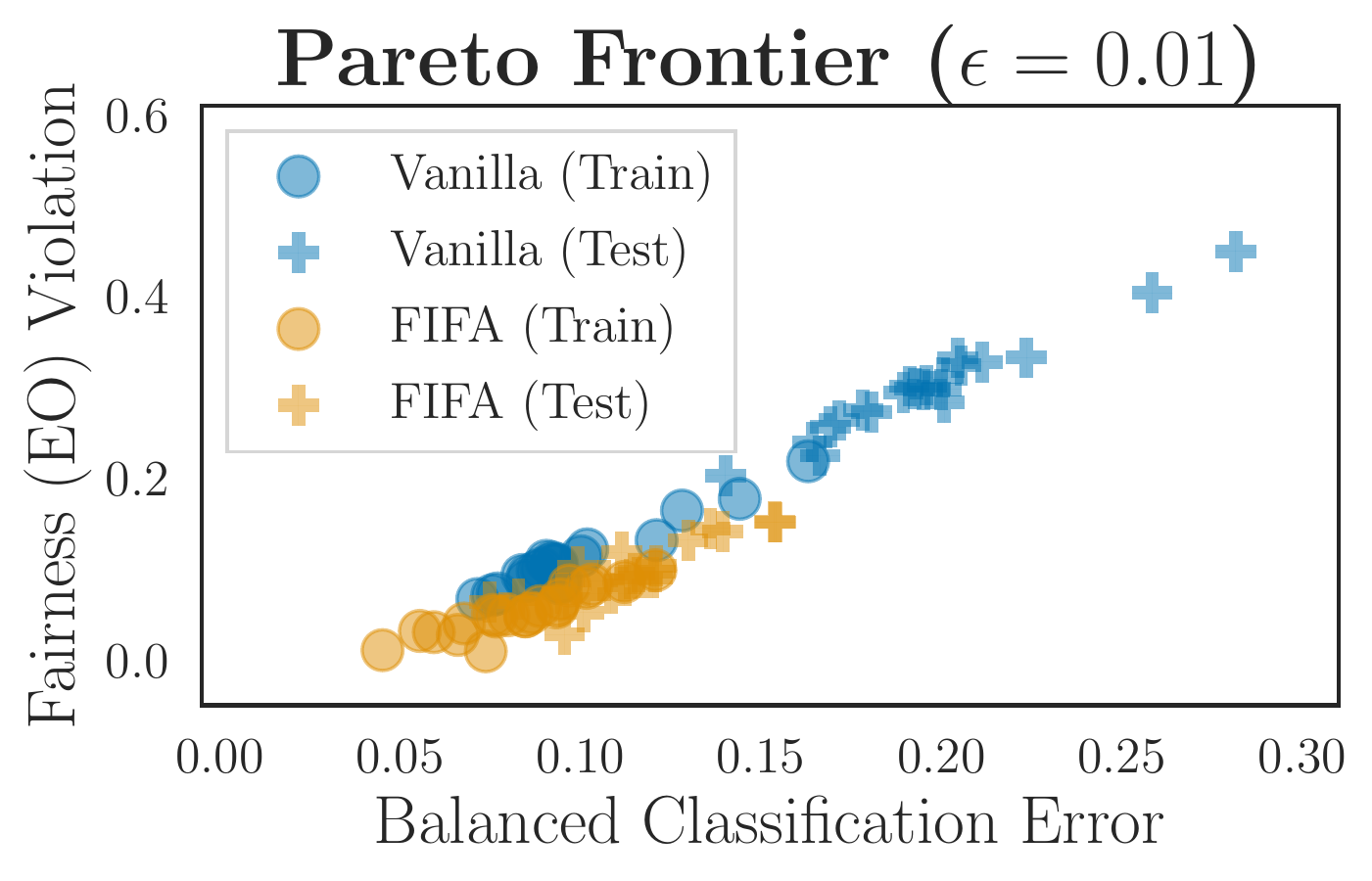}
	    \caption{$\epsilon=0.01$.} \label{fig:celeba:pareto:0.01}
	\end{subfigure}~
	\begin{subfigure}[t]{0.33\textwidth}
	    \centering
	    \includegraphics[width=\linewidth]{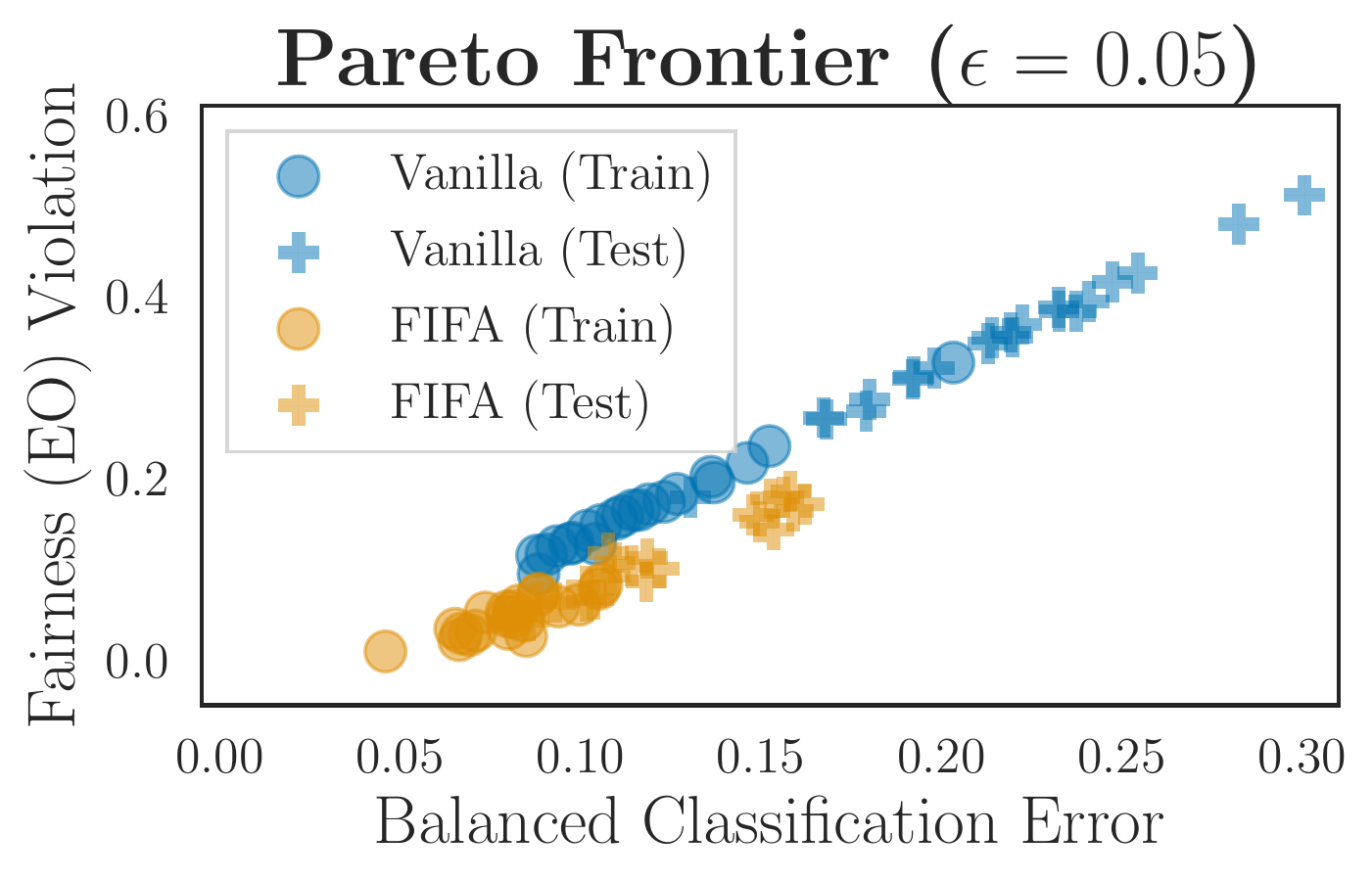}
	    \caption{$\epsilon=0.05$.} \label{fig:celeba:pareto:0.05}
	\end{subfigure}%
    \begin{subfigure}[t]{0.33\textwidth}
	    \centering
	    \includegraphics[width=\linewidth]{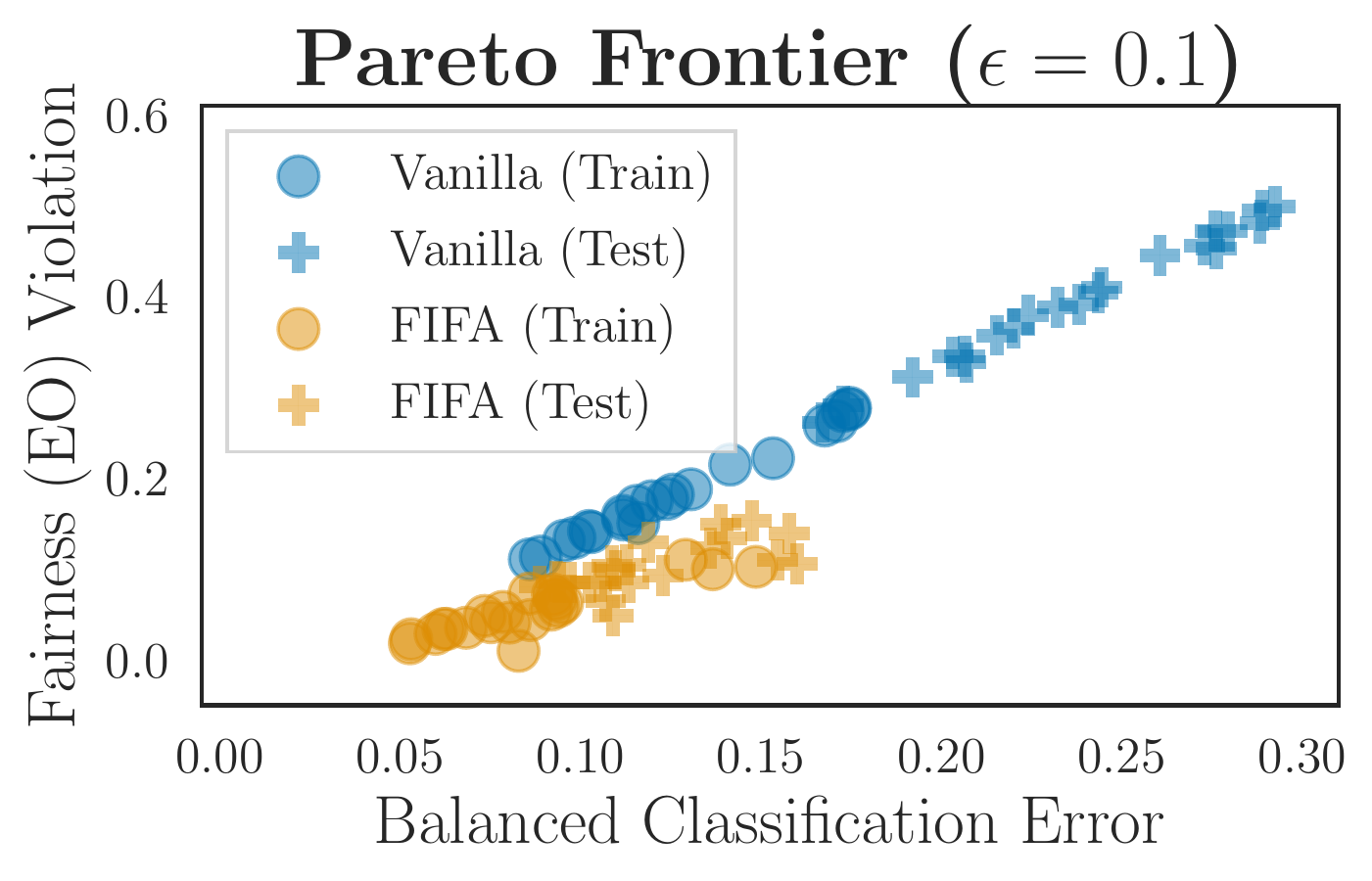}
	    \caption{$\epsilon=0.10$.} \label{fig:celeba:pareto:0.1}
	\end{subfigure}%
	\caption{Pareto frontiers of the balanced loss ($\Lbal$) and fairness loss ($\Lfair$) of CelebA using ResNet-18 with grid search combined with FIFA and vanilla softmax-cross-entropy loss respectively. Best-performing hyper-parameters from Table~\ref{tab:celeba} are used, where each configuration is repeated $20$ times independently. Here blue and orange markers correspond to vanilla and FIFA respectively, and circular and cross markers correspond to training and testing metrics respectively. We observe that our FIFA method is effective in significantly lowering the Pareto frontier comparing with the vanilla method, implying that FIFA mitigates fairness generalization issues as seen in
	Figure~\ref{fig:teaser}.}
	\label{fig:celeba:pareto}
\end{figure*}

\subsection{Effectiveness of FIFA on smaller datasets and models}
We use logistic regression (implemented as a one-layer neural net) for the AdultIncome
and DutchConsensus datasets  with similar sweeping procedure are similar to those described in Section~\ref{sec:exp:large} except that we set
$\delta_{0, \text{Female}}, \delta_{1, \text{Female}} \in [0,0.01]$,
and $\delta_{0, \text{Male}}=\delta_{1, \text{Male}}=0$ for the AdultIncome dataset 
and
$\delta_{0, \text{Male}}, \delta_{1, \text{Female}} \in [0,0.01]$,
and $\delta_{0, \text{Female}}=\delta_{1, \text{Male}}=0$ for the DutchConsensus Dataset.

\paragraph{Results.} We tabulate the best-performing models (in terms of test combined loss) among sweeps in Table~\ref{tab:adultincome} and include more details in the Appendix. From Table~\ref{tab:adultincome}, we observe the same message as we did in Section~\ref{sec:exp:large}, namely, our method FIFA outperforms vanilla on both dataset across three different tolerance parameter $\eps$, illustrating the effectiveness of FIFA on boosting fairness generalization on smaller datasets and models as well. Nonetheless, since the datasets are much simpler in this case, the improvements are not as large as in the case for CelebA and larger models.

    
    
    
    
    

\section{Conclusions}
Generalization (especially in over-parameterized models)
has always been an important and difficult problem in machine learning research. In this paper, we set out the first exposition in the study of the generalization of \emph{fairness constraints} that has  previously been overlooked. Our theoretically-backed FIFA approach is shown to mitigate poor fairness generalization observed in vanilla models large or small. In our paper, the analysis is based on margin upper bounds. We will leave a more fine-grained analysis of the margins to the future work. Our work is motivated by the societal impact of improving fairness across subpopulations. Methods like FIFA should be used with care and may not fully mitigate bias; better training data is still needed in many settings. 

\section*{Acknowledgements}
The research of Z.~Deng is supported by the Sloan Foundation grants, the NSF grant 1763665, and the Simons Foundation Collaboration on the Theory of Algorithmic Fairness. J.~Zhang is supported by ONR Contract N00015-19-1-2620.
W.J.~Su is supported in part by NSF through CCF-1934876, an Alfred Sloan Research Fellowship, and the Wharton Dean’s Research Fund. 
We would like to thank
Dan Roth and the Cognitive Computation Group at the University of Pennsylvania for stimulating
discussions and for providing computational resources.
J.~Zhang also thanks Chenwei Wu from Duke University for helpful discussions.

{\small
\bibliography{ref,zhun}
}




\setcounter{section}{1}\setcounter{equation}{0}
\numberwithin{equation}{section}

\clearpage
\appendix

\noindent{\Large \textbf{Appendix}}

\section{Omitted derivation}
In this section, we will talk about several missing details in the main context.

\subsection{Extension to multi-classes and multi-groups for equalized odds}\label{app:1}
We discussed the case that $\cY=\{0,1\}$ and $\cA=\{a_1,a_2\}$ in the main context. In this subsection, we discuss the extension to multi-classes and multi-groups.

First, we extend the case to $\cY=\{0,1\}$ and $|\cA|\ge 2$. In that case, we can consider the constraints $$\sum_{a,a'\in\cA,a\neq a'}\sum_{i\in\cY}\left|\bP(h(X)=i|Y=i,A=a)-\bP(h(X)=i|Y=i,A=a')\right|,$$
regardless of the order of $(a,a')$, and there are $|\cA|\choose 2$ pairs of $(a,a')$'s.

Thus, our performance criteria $\cM_{\text{multi-groups}}[f]$ can be taken as:
\begin{equation*}
 \Lbal[f]+\alpha \sum_{a,a'\in\cA,a\neq a'}\sum_{i\in\cY}\left|\bP(h(X)=i|Y=i,A=a)-\bP(h(X)=i|Y=i,A=a')\right|.
 \end{equation*}
 
 Then, via Lemma~\ref{lm:complete}, for all $f\in\cF$
\begin{align*}
\cM_{\text{multi-groups}}[f]&\lesssim\sum_{i\in\cY}\frac{1}{\gamma_i}\sqrt{\frac{C(\cF)}{n_i}}+ \sum_{i\in\cY,a\in\cA}\frac{\alpha (|A|-1)}{\gamma_{i,a}}\sqrt{\frac{C(\cF)}{n_{i,a}}}\\
&\le \sum_{i\in\cY}\frac{1}{\gamma_i}\sqrt{\frac{C(\cF)}{n_i}}+ \sum_{i\in\cY,a\in\cA}\frac{\alpha(|A|-1)}{\gamma_{i}}\sqrt{\frac{C(\cF)}{n_{i,a}}}.
\end{align*}
We overload the notation $\alpha$, then we have 
\begin{equation}\label{eq:appmultigroup}
\cM_{\text{multi-groups}}[f]\lesssim\sum_{i\in\cY}\frac{1}{\gamma_i}\sqrt{\frac{C(\cF)}{n_i}}+ \sum_{i\in\cY,a\in\cA}\frac{\alpha}{\gamma_{i}}\sqrt{\frac{C(\cF)}{n_{i,a}}}.
\end{equation}
Notice the difference between Eq.~\eqref{eq:appmultigroup} and Eq.~\eqref{eq:M} is that in Eq.~\eqref{eq:appmultigroup}, $|\cA|\ge 2$. Thus, by similar proof as in Theorem~\ref{thm:appeomargin}, we can obtain 
$$\gamma_0/\gamma_1=\tilde{n}_0^{-1/4}/\tilde n_1^{-1/4},$$
where the adjusted sample size 
$$\tilde n_i=\frac{n_i\Pi_{a}n_{i,a}}{(\sqrt{\Pi_{a}n_{i,a}}+\alpha\sum_{j\in\cA}\sqrt{n_i\Pi_{a\in\cA^{\backslash j}}n_{i,a}})^2},$$
for $i\in \{0,1\}$.

Given the results above, for multiple classes with multiple groups, for $i,j\in\cY$, 
$$\gamma_i/\gamma_j=\tilde{n}_i^{-1/4}/\tilde n_j^{-1/4},$$
where the adjusted sample size $\tilde n_i=\frac{n_i\Pi_{a}n_{i,a}}{(\sqrt{\Pi_{a}n_{i,a}}+\alpha\sum_{j\in\cA}\sqrt{n_i\Pi_{a\in\cA^{\backslash j}}n_{i,a}})^2}$
for $i\in \cY$.

\paragraph{FIFA for multi-classes and multi-groups.} We will demonstrate how to apply the above motivations to design better margin losses. Consider a logits-based loss
\begin{equation*}
\ell((x,y);f)=\ell(f(x)_y, \{f(x)_i\}_{i\in\cY\backslash y}),
\end{equation*}
which is non-increasing with respect to its first coordinate if we fix the second coordinate.

Our flexible imbalance-fairness-aware (FIFA) approach modifies the above losses during training
\begin{equation}
\lfifa((x,y,a);f)=\ell(f(x)_y-\Delta_{y,a}, \{f(x)_i\}_{i\in\cY\backslash y})
\end{equation}
where $\Delta_{i,a}=C/\tilde{n}_i^{1/4}+\delta_{i,a}$, and $\delta_{i,a}\ge 0$. The specific assignment of $\delta_{i,a}\ge 0$ is described in Section~\ref{subsec:delta}.

\subsection{Assignment of $\delta_{i,a}$ for multi-groups}\label{subsec:delta}
In this subsection, we describe how to choose $\delta_{i,a}$ for multiple demographic groups. Recall in Section~\ref{app:1}, for multiple demographic groups,
\begin{align}\label{eq:multi-group}
\cM_{\text{multi-groups}}[f]&\lesssim\sum_{i\in\cY}\frac{1}{\gamma_i}\sqrt{\frac{C(\cF)}{n_i}}+ \sum_{i\in\cY,a\in\cA}\frac{\alpha}{\gamma_{i,a}}\sqrt{\frac{C(\cF)}{n_{i,a}}}.
\end{align}
Given $\gamma_{i,a}\ge \gamma_i$, similar as in the main context, we can take $\gamma_{i,a}=\gamma_i+\delta_{i,a}$, where $\delta_{i,a}\ge 0$ are tuning parameters.

Assume there are $k$ groups. Let us first ordering $|S_{i,a}|$ by in a decreasing order. Without loss of generality, $|S_{i,a_1}|\ge|S_{i,a_2}|\ge\cdots\ge |S_{i,a_k}| $. Thus, when tune the parameters $\delta_{i,a}$'s, we set $\delta_{i,a_1}=0$ (or we can randomly choose other $a$'s to set $\delta_{i,a}=0$ if there are ties and $|S_{i,a}|=|S_{i,a_1}|$, but for simplicity, we ignore this case), and we make sure $\delta_{i,a_1}\le \delta_{i,a_2}\le \cdots \delta_{i,a_k}$. This is optimal in the sense that if there are $k$ constants $0=\delta_1\le\delta_2\le \delta_3\le \cdots\le \delta_k$, then 
\begin{equation*}
\sum_{i\in\cY,a_j\in\cA}\frac{\alpha}{\gamma_i+\delta_{i,a_j}}\sqrt{\frac{C(\cF)}{n_{i,a_j}}} \le \sum_{i\in\cY,a_j\in\cA}\frac{\alpha}{\gamma_i+\delta_{i,a_{\sigma(j)}}}\sqrt{\frac{C(\cF)}{n_{i,a_j}}},
\end{equation*}
where $\sigma(\cdot)$ is a permutation. In other words, our way to assign $\delta_{i,a}$'s can make the upper boupnd on RHS of  Eq.~\eqref{eq:multi-group} optimal. This is a direct application of rearrangement inequality, see Lemma~\ref{lm:rearrange}.

\begin{lemma}\label{lm:rearrange}
For $x_1\le x_2\le\cdots\le x_k$, $y_1\le y_2\le\cdots\le y_k$, any permutation $\sigma(\cdot)$
$$x_{k}y_1+x_{k-1}y_2+\cdots+x_{1}y_k\le x_{\sigma(1)}y_1+x_{\sigma(2)}y_2+\cdots+x_{\sigma(k)}y_k\le x_1y_1+x_2y_2+\cdots+x_ky_k.$$
\end{lemma}
\subsection{Derivation for other fairness notions}
In this subsection, we consider theory-inspired derivation for other fairness notions. For simplicity, we still focus on the case that $\cY=\{0,1\}$ (this is necessary for EqOpt) and $\cA=\{a_1,a_2\}$. Given the derivation in this subsection, we can further derive FIFA for other fairness notions as in Section~\ref{app:1}.

\paragraph{Equalized opportunity.}   Specifically, for EqOpt, the aim is:
 \begin{align*}
  &\min_f \Lbal[f]\\
  \text{s.t.}~\forall y\in\cY, a\in \cA,~~\bP(h(X)&=1|Y=1,A=a)=\bP(h(X)=1|Y=1),
\end{align*}
This is a simple version of EO in some sense. Directly using the derivation in Section~\ref{app:1}, we have $$\gamma_0/\gamma_1=n_0^{-1/4}/\tilde n_1^{-1/4},$$
where the adjusted sample size 
$$\tilde n_1 =\frac{n_1n_{1,a_1}n_{1,a_2}}{(\sqrt{n_{1,a_1}n_{1,a_2}}+\alpha(\sqrt{n_1n_{1,a_2}}+\sqrt{n_1n_{1,a_1}}))^2}.$$

\paragraph{Demographic parity.} Similar to EO, for DP, the optimization aims to:
\begin{align*}
  &\min_f \Lbal[f]\\
  s.t.~\forall y\in\cY, a\in \cA,~~\bP&(h(X)=y|A=a)=\bP(h(X)=y).
\end{align*}
In this setting, we no longer can expect all the training samples are perfectly classified and fairness constraints violation is perfectly satisfied because there exists fairness and accuracy trade-off in training phase for DP. However, in real application, people always adopt relaxation in fairness constraints, i.e. $|\bP(h(X)=y|A=a)-\bP(h(X)=y)|<\epsilon$ for some $\epsilon>0$ (if there are only two groups, one can alternatively use $|\bP(h(X)=y|A=a_1)-\bP(h(X)=y)|A=a_2|\le \epsilon$). When $\epsilon$ is large enough (or $\hat\bP(Y=1|A=a_1)$ is close to $\hat\bP(Y=1|A=a_2)$), if we use suitable models, similar as in the EO setting, we would expect all the training examples are classified perfectly while satisfying $|\hat\bP(h(X)=y|A=a_1)-\hat\bP(h(X)=y)|A=a_2|\le \epsilon$. Then, we can use similar techniques to characterize a trade-off between margins.



Specifically, simple calculation leads to $\sum_{i\in\cY}|\bP(h(X)=i|A=a_1)-\bP(h(X)=i|A=a_2)|\le \sum_{j\in\{1,2\}}\sum_{i\in \{0,1\}}\bP(Y=i|A=a_j)\cL_{i,a_j}[f]+I$, where $I$ is a term not related to $f$. Thus, our \textbf{\textit{optimization objective}} (not performance criterion) for DP can be taken as: 
\begin{equation}\label{eq:cri}
 \Lbal[f]+ \alpha\sum_{i,a}\bP(Y=i|A=a)\cL_{i,a}[f],
\end{equation}
for weight $\alpha$. We can use training data to estimate $\bP(Y=i|A=a)$. For simplicity, we can also use $\Lbal[f]+ \alpha\sum_{i,a}\cL_{i,a}[f]$, which shares the same upper bound as in Theorem~\ref{thm:eomargin} and also implies $\gamma_0/\gamma_1=\tilde{n}_0^{-1/4}/\tilde n_1^{-1/4},$ that will also be used in the experiments in later sections. We should also remark that when $\epsilon$ is too small, our method may not work, this can also be reflected in Table~\ref{tab:celeba:dp}, when $\epsilon=0.01$.

\subsection{Omitted proofs}
\subsubsection{Proof of Theorem~\ref{thm:eomargin}}
\begin{theorem}[Restatement of Theorem~\ref{thm:eomargin}]\label{thm:appeomargin}
With high probability over the randomness of the training data, for $\cY=\{0,1\}$, $\cA=\{a_1,a_2\}$, and for some proper complexity measure of class $\cF$, for any $f\in\cF$,
\begin{align*}
\begin{split}
\cM[f]\lesssim\sum_{i\in\cY}\frac{1}{\gamma_i}\sqrt{\frac{C(\cF)}{n_i}}+ \sum_{i\in\cY,a\in\cA}\frac{\alpha}{\gamma_{i,a}}\sqrt{\frac{C(\cF)}{n_{i,a}}}
\le \sum_{i\in\cY}\frac{1}{\gamma_i}\sqrt{\frac{C(\cF)}{n_i}}+ \sum_{i\in\cY,a\in\cA}\frac{\alpha}{\gamma_{i}}\sqrt{\frac{C(\cF)}{n_{i,a}}},
\end{split}
\end{align*}
where $\gamma_i$ is the margin of the $i$-th class's sample set $S_i$ and $\gamma_{i,a}$ is the margin of demographic subgroup's sample set $S_{i,a}$.
\end{theorem}

This following lemma is the key lemma we will use. Let us define the empirical Rademacher complexity of $\cF$ of subgroup/class margin on $S_*$ as
$$\hat{\cR}_i(\cF)=\frac{1}{n_i}\bE_{\xi}[\sup_{f\in\cF}\sum_{j\in S_i}\xi_j[f(x_j)_i-\max_{i'\neq i}f(x_j)_{i'}]],$$

$$\hat{\cR}_{i,a}(\cF)=\frac{1}{n_{i,a}}\bE_{\xi}[\sup_{f\in\cF}\sum_{j\in S_{i,a}}\xi_j[f(x_j)_i-\max_{i'\neq i}f(x_j)_{i'}]],$$
where $\xi_j$ is i.i.d. drawn from a uniform distribution $\{-1,1\}$.

\begin{lemma}\label{lm:complete}
Let $\hat{\cL}_{\gamma,i}[f]=\bP_{x\sim \hat{\cP}_i}(\max_{j\neq i}f(x)_j>f(x)_i-\gamma)$ and $\hat{\cL}_{\gamma,(i,a)}[f]=\bP_{x\sim \hat{\cP}_{i,a}}(\max_{j\neq i}f(x)_j>f(x)_i-\gamma)$. With probability at least $1-\delta$ over the the randomness of the training data, for some proper complexity measure of class $\cF$, for any $f\in\cF$, $*\in\{i,(i,a)|i\in \cY,a\in\cA\}$, and all margins $\gamma>0$
\begin{equation}\label{eq:M}
\cL_*[f]\lesssim \hat{\cL}_{\gamma,*}[f]+ \frac{1}{\gamma}\hat{\cR}_*(\cF)+\epsilon_*(n_*,\delta,\gamma_*),
\end{equation}
where $\hat{\cR}_*(\cF)$ is the empirical Rademacher complexity of $\cF$ of subgroup/class margin on training dataset corresponding to index set $S_{*}$, which can be further upper bnounded by $\sqrt{\frac{C(\cF)}{n_*}}$. Also, $\epsilon_*(n_*,\delta,\gamma_*)$ is usually a low-order term in $n_*$
\end{lemma}
\begin{proof}
This is a direct application of the standard margin-based generalization bound in \citep{kakade2008complexity}.
\end{proof}

\paragraph{Proof of Theorem~\ref{thm:appeomargin}.} Notice that all the training samples are classified perfectly by $h$, not only $\bP_{(x,y)\sim\hat \cP_{\bal}}(h(x)\neq y)=0$ is satisfied, we also have that $\bP_{(x,y)\sim\hat{\cP}_{i,a_j}}(h(x)\neq y)=0$ for all $i\in\cY$ and $a_j\in\cA$. We remark here that $\bP(h(X)=i|Y=i,A=a)=1-\bP_{(x,y)\sim\cP_{i,a}}(h(x)\neq y)=1-\bP(f(x)_y<\max_{j\neq y}f(x)_j)$.
\begin{align*}
&\Lbal[f]+\alpha \sum_{i\in\cY}\left|\bP(h(X)=i|Y=i,A=a_1)-\bP(h(X)=i|Y=i,A=a_2)\right|\\
&\le \Lbal[f]+\alpha (\bP_{(x,y)\sim\cP_{i,a_1}}(h(x)\neq y)+\bP_{(x,y)\sim\cP_{i,a_2}}(h(x)\neq y)).
\end{align*}
Then, plug in Lemma~\ref{lm:complete}, the proof is complete.

\subsubsection{Optimization of $\gamma_0/\gamma_1$}

\begin{theorem}
For binary classification, let $\cF$ be a class of neural networks with a bias term, i.e. $\cF=\{f+b\}$ where $f$ is a neural net function and $b\in\bR^2$ is a bas, with Rademacher complexity upper bound $\hat{\cR}_*(\cF)\le\sqrt{\frac{C(\cF)}{n_*}}$. Suppose some classifier $f\in\cF$ can achieve a total sum of margins $\gamma'_0+\gamma'_1=\beta$ with $\gamma'_0,\gamma'_1>0.$ Then, there exists a classifier $f^*\in\cF$ with margin ratio
$$\gamma^*_0/\gamma^*_1=\tilde{n}_i^{-1/4}/\tilde n_j^{-1/4},$$
where the adjusted sample size $\tilde n_i=\frac{n_i\Pi_{a}n_{i,a}}{(\sqrt{\Pi_{a}n_{i,a}}+\alpha\sum_{j\in\cA}\sqrt{n_i\Pi_{a\in\cA^{\backslash j}}n_{i,a}})^2}$
for $i\in \cY$.
\end{theorem}
\begin{proof}
This can directly follow simliar proof in Theorem 3 in \citep{cao2019learning}. The only difference is that we need to solve 
$$\min_{\gamma_0+\gamma_1=\beta}\sum_{i\in\cY}\frac{1}{\gamma_i}\sqrt{\frac{1}{n_i}}+ \sum_{i\in\cY,a\in\cA}\frac{\alpha}{\gamma_{i}}\sqrt{\frac{1}{n_{i,a}}},$$
which yields the final result.
\end{proof}

\subsubsection{Proof of Example~\ref{ex:example}.}

\begin{example}[Restatement of Example~\ref{ex:example}]
Given function $f$ and set $S$, let $$\dist(f,S)=\min_{x,s\in S}\|f(x)-s\|_2.$$ Consider two classifiers $\tilde{f}, f \in \cF$ such that $$\dist(\tilde f, S_0)/\dist(\tilde f, S_1)=\tilde{n}_0^{-1/4}/\tilde n_1^{-1/4}$$ and $\dist(f', S_0)/\dist(f', S_1)=n_0^{-1/4}/ n_1^{-1/4}$. Suppose  $\|\beta^*\|\gg \sqrt{p\log n}$, $\|\mu_i\|<C$, $(\mu_1^*-\mu_2^*)^\top\beta=0$,  
and $\pi_{1,a_2}\le c_1 \pi_{1,a_1}$ for a sufficiently small $c_1>0$, then when $n_0, n_1$ are sufficiently large, with high probability we have $\cM[\tilde f] < \cM[f].$
\end{example}

\begin{proof} Recall that $\tilde n_i=\frac{n_i\Pi_{a\in\cA}n_{i,a}}{(\sqrt{\Pi_{a\in\cA}n_{i,a}}+\alpha\sum_{a\in\cA}\sqrt{n_in_{i,a}})^2}$ for $i\in \{0,1\}$, and our training data follow distribution: $x\mid y=0 \sim \sum_{i=1}^{2}\pi_{0,a_i} \cN_p(\mu_i, I),~~x\mid y=1 \sim \sum_{i=1}^{2}\pi_{1,a_i} \cN_p(\mu_i+\beta^*, I).$ 

\begin{align*}
 \cM[f]=&\frac{1}{2}\bP(h(X)=1|Y=0)+\frac{1}{2}\bP(h(X)=0|Y=1)\\
 &+\alpha \sum_{i\in\cY}\left|\bP(h(X)=i|Y=i,A=a_1)-\bP(h(X)=i|Y=i,A=a_2)\right|\\
 =&\frac{1}{2}\sum_{i=1}^{2}\pi_{0,a_i}\Phi(\frac{\beta^{*\top}\mu_i-c}{\|\beta^*\|})+\frac{1}{2}\sum_{i=1}^{2}\pi_{1,a_i}\Phi(\frac{c-\beta^{*\top}\mu_i-\|\beta^*\|^2}{\|\beta^*\|})\\
 &+\alpha\cdot|\Phi(\frac{\beta^{*\top}\mu_0-c}{\|\beta^*\|})-\Phi(\frac{\beta^{*\top}\mu_1-c}{\|\beta^*\|})|+\alpha\cdot|\Phi(\frac{\beta^{*\top}\mu_0+\|\beta^*\|^2-c}{\|\beta^*\|})-\Phi(\frac{\beta^{*\top}\mu_1+\|\beta^*\|^2-c}{\|\beta^*\|})|
\end{align*}

For different margin ratio $\gamma$, we have $c=\mu_1^\top\beta^*+\frac{1}{1+\gamma}\|\beta^*\|^2+O_P(\sqrt{p\log n})$, where the $O_P(\sqrt{p\log n})$ term accounts for the variation of random samples and is based on the fact that $\|Z\|^2\sim\chi_p^2$ if $Z\sim \mathcal N_p(0,I)$ and $\max Z_i=O_P(p\log n)$ if $Z_1,...,Z_n\stackrel{i.i.d.}{\sim}\chi_p^2$.

Using the fact that $\|\beta^*\|\gg \sqrt{p\log n}$ and $\|\mu_i\|<C$, we then have $
c=(\frac{1}{1+\gamma}+o_P(1))\|\beta^*\|^2.
$

Similarly, we have  $$
  \cM[f]=\Phi(-\frac{c}{\|\beta^*\|})+\Phi(\frac{c-\|\beta^*\|^2}{\|\beta^*\|})+ o(1)
 $$

Then let's consider the two ratios $\gamma=n_0^{-1/4}/n_1^{-1/4}$ and $\tilde\gamma=\tilde n_0^{-1/4}/\tilde n_1^{-1/4}$.

In the following we compute $\tilde n_0$ and $\tilde n_1$:

When
$\pi_{1,a_2}\le c_1 \pi_{1,a_1}$
$\tilde n_1=\frac{n_1\Pi_{a\in\cA}n_{1,a}}{(\sqrt{\Pi_{a\in\cA}n_{1,a}}+\alpha\sum_{a\in\cA}\sqrt{n_1n_{1,a}})^2}\in (0.9 n_1, n_1)$ 

When
$\pi_{0,a_2}= \pi_{0,a_1}$
$\tilde n_0=\frac{n_0\Pi_{a\in\cA}n_{0,a}}{(\sqrt{\Pi_{a\in\cA}n_{0,a}}+\alpha\sum_{a\in\cA}\sqrt{n_0n_{0,a}})^2}=(\frac{1}{1+2\sqrt{2}})^2 n_0.$

 As a result, we have $\frac{\tilde \gamma}{\gamma}\in[1.9,2]$. When the data is imbalanced such that $\gamma=(\frac{n_1}{n_0})^{1/4}>1$, we have $0<\frac{1}{1+\tilde\gamma}<\frac{1}{1+\gamma}<1/2$, and consequently
$$
\Phi(-\frac{\frac{1}{1+\tilde\gamma}\|\beta^*\|^2}{\|\beta^*\|})+\Phi(\frac{\frac{1}{1+\tilde\gamma}\|\beta^*\|^2-\|\beta^*\|^2}{\|\beta^*\|})<\Phi(-\frac{\frac{1}{1+\tilde\gamma}\|\beta^*\|^2}{\|\beta^*\|})+\Phi(\frac{\frac{1}{1+\tilde\gamma}\|\beta^*\|^2-\|\beta^*\|^2}{\|\beta^*\|}).
$$
Additionally, we have that when $\|\beta^*\|\to\infty$, the second term in the $ \cM[f]$, the fairness violation error is $o(1)$. 
Combining all the pieces, we have $\cM[\tilde f] < \cM[f].$

\end{proof}
\subsubsection{Proof of Theorem~\ref{thm:convergence}.}
\begin{theorem}[Restatement of Theorem~\ref{thm:convergence}]
Let $\rho=\max_f\|M\hat{\mu}^{\text{new}}(f)-\hat c\|_\infty$. For $\eta=\nu/(2\rho^2B)$, the modified \textbf{ExpGrad} will return a $\nu$-approximate saddle point of $L^{\text{new}}$ in at most $4\rho^2B^2\log(|\cK|+1)/\nu^2$ iterations.
\end{theorem}
\begin{proof}
We consider an extended version of $h$ in \citep{agarwal2018reductions}, which is a function of (x,y,a) instead of just be a function of $x$. $h: (x,y,a)\mapsto \{0,1\}$. Notice that $\mu(h)$ also satisfies the requirement in \citep{agarwal2018reductions} with the extend version $h$. Thus, directly by classic result of in Freund \& Scapire (1996) and theorm $1$ in \citep{agarwal2018reductions}, the result follows.
\end{proof}

\subsection{Combination with other algorithms}
 As we stated, the algorithm stated in the main context is just one of the examples that can be combined with our approach. FIFA can also be applied to many other popular algorithms such as fair representation \citep{madras2018learning}. We here show how to combine with fair representation.

In \citep{madras2018learning}, there are several parts, an encoder $\rho$, an adversary $v$, a decoder $k$ and a predictor $g$. The optimization is:

$$\min_{g,\rho,k}\max_v \bE_{X,Y,A} L(g,\rho,k,v),$$
where 
\begin{align*}
L(g,\rho,k,v)=\lambda_1 \ell_c(g(\rho(X)),Y)+\lambda_2\ell_{\text{dec}}(k(\rho(X),A),X)+\lambda_3\ell_{\text{adv}}(v(\rho(X),A)),
\end{align*}
for cross entropy loss $\ell_c$, decoding loss $\ell_{\text{dec}}$, and adversary loss $\ell_{\text{adv}}$. We can modify the cross entropy loss to $\ell_c$ to $\ell_{\text{FIFA}}$. So, 
\begin{align*}
L_{\text{FIFA}}(g,\rho,k,v)=\lambda_1 \ell_{\text{FIFA}}(g(\rho(X)),Y)+\lambda_2\ell_{\text{dec}}(k(\rho(X),A),X)+\lambda_3\ell_{\text{adv}}(v(\rho(X),A)).
\end{align*}

Actually, for $\ell_{\text{adv}}$, we can similarly modify for indices, but it is a little complicated and notation heavy, so we omit it here.

\section{Implementation details}

\begin{table*}[h]
    \scriptsize\centering
    \resizebox{\textwidth}{!}{%
    \begin{tabular}{lrrrrrrrrrrrr}
\toprule
{\bf Data} & \multicolumn{4}{c}{\bf AdultIncome} & \multicolumn{4}{c}{\bf CelebA} & \multicolumn{4}{c}{\bf DutchConsensus} \\
{\bf Label} & \multicolumn{2}{c}{\texttt{-}} & \multicolumn{2}{c}{\texttt{+}} & \multicolumn{2}{c}{\texttt{-}} & \multicolumn{2}{c}{\texttt{+}} & \multicolumn{2}{c}{\texttt{-}} & \multicolumn{2}{c}{\texttt{+}} \\
{\bf Gender} &      Female &   Male & Female &  Male & Female &   Male & Female &  Male &         Female &   Male & Female &   Male \\
\midrule
Train &        9592 &  15128 &   1179 &  6662 &  71629 &  66874 &  22880 &  1387 &          69117 &  59608 &   7991 &  15064 \\
Test  &        4831 &   7604 &    590 &  3256 &   9767 &   7535 &   2480 &   180 &          17231 &  15006 &   1912 &   3796 \\
\bottomrule
\end{tabular}
    }
\caption{Training and testing sample sizes for CelebA and AdultIncome datasets across labels (Label) and sensitive attributes (Gender). 
    }
    \label{tab:datasummary}
\end{table*}
    We use the official train-test split for the CelebA dataset. For AdultIncome and DutchConsensus,
    we use the
    \texttt{train\_test\_split} procedure of the \texttt{scikit-learn} package with training-test
    set ratio of $0.8$ and random seed of \texttt{1}
    to generate the training and test set.
    We tabulate the sizes for subgroups in Table~\ref{tab:datasummary}.
    
\paragraph{Details on CelebA.}
    We use the same pre-processing steps
    as in \citep{sagawa2020investigation} to
    crop the images in CelebA into $224\times 224\times 3$ and perform the same $z$-normalization for both training and test set. We use ResNet-18 models
    for training with the last layer being replaced to the \texttt{NormLinear} layer used by \citep{cao2019learning} that ensures the input as well as the columns of the wight matrix (with $2$ rows corresponding to each label class) has norm $1$. This ensures our adjustments on the logits are comparable. We use the Adam optimizer with learning rate $1\times 10^{-4}$ and weight decay $5\times 10^{-5}$ to train these models with stochastic batches of sizes $128$. We performed pilot experiments and learnt that under this configuration the models usually converges within the first $1500$ iterations in terms of training losses and thus we fix the training time as $8000$ iterations which corresponds to roughly four epochs.

\paragraph{AdultIncome and DutchConsensus.}
    AdultIncome and DutchConsensus are two relatively
    smaller datasets that have been used for benchmarking for various fair classification algorithms
    such as \citep{agarwal2018reductions}. We convert
    all categorical variables to dummies and use the standard $z$-normalization to pre-process the data.
    There are $107$ features in AdultIncome and $59$
    in DutchConsensus, both counting the senstive attribute, gender. We intend to use these datasets
    to test smaller models such as logistic regression,
    and we implement it as a one-layer neural net
    for consistency concerns, which is trained using
    full-batch gradient descent using Adam with learning rate $1\times 10^{-4}$ and weight decay $5\times 10^{-5}$ for $10000$ epochs. All models
    converge after this training measured by the training
    metrics.
    Although the exact pre-processing
    procedures for these two datasets are not available
    in \citep{agarwal2018reductions}, we found that
    on vanilla models under both GridSearch and ExponentiatedGradient methods, the training and 
    test performance (measured by total accuracy and
    fairness violation) are comparable with those
    reported in \citep{agarwal2018reductions}.
    
\paragraph{Computational resource considerations.}
    We perform all experiments on \texttt{NVIDIA}
    GPUs \texttt{RTX 2080 Ti}. Each experiment
    on CelebA usually takes less than two hours (clocktime) and each experiment on AdultIncome and DutchConsensus takes less than ten minutes. 

\section{Additional experimental results}

\subsection{A trajectory analysis on CelebA}
\begin{figure*}[t]
    \centering
    \begin{subfigure}[t]{0.45\textwidth}
	    \centering
	    \includegraphics[width=\linewidth]{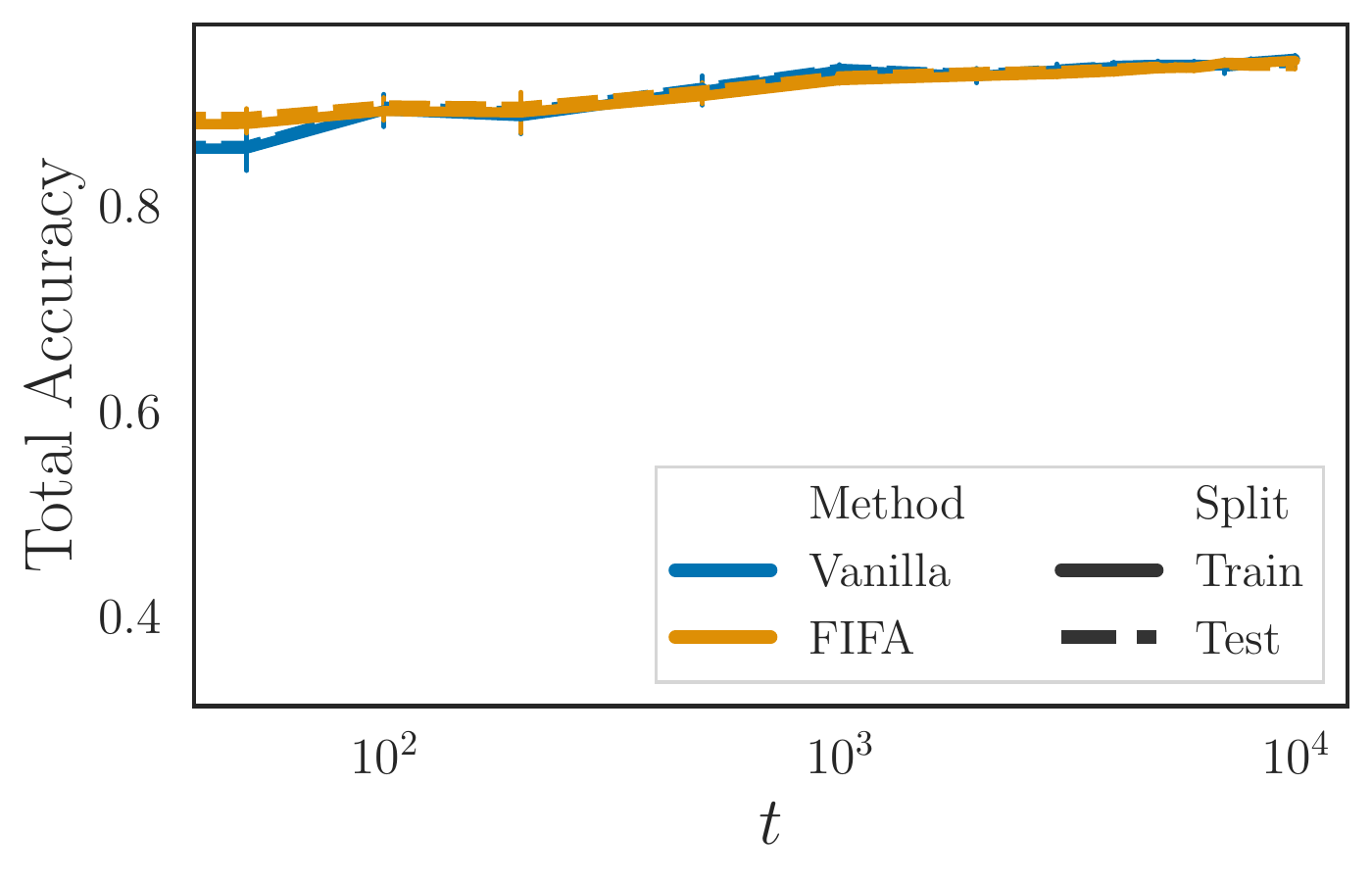}
	    \caption{Total accuracy.} \label{fig:celeba:traj:total}
	\end{subfigure}~
	\begin{subfigure}[t]{0.45\textwidth}
	    \centering
	    \includegraphics[width=\linewidth]{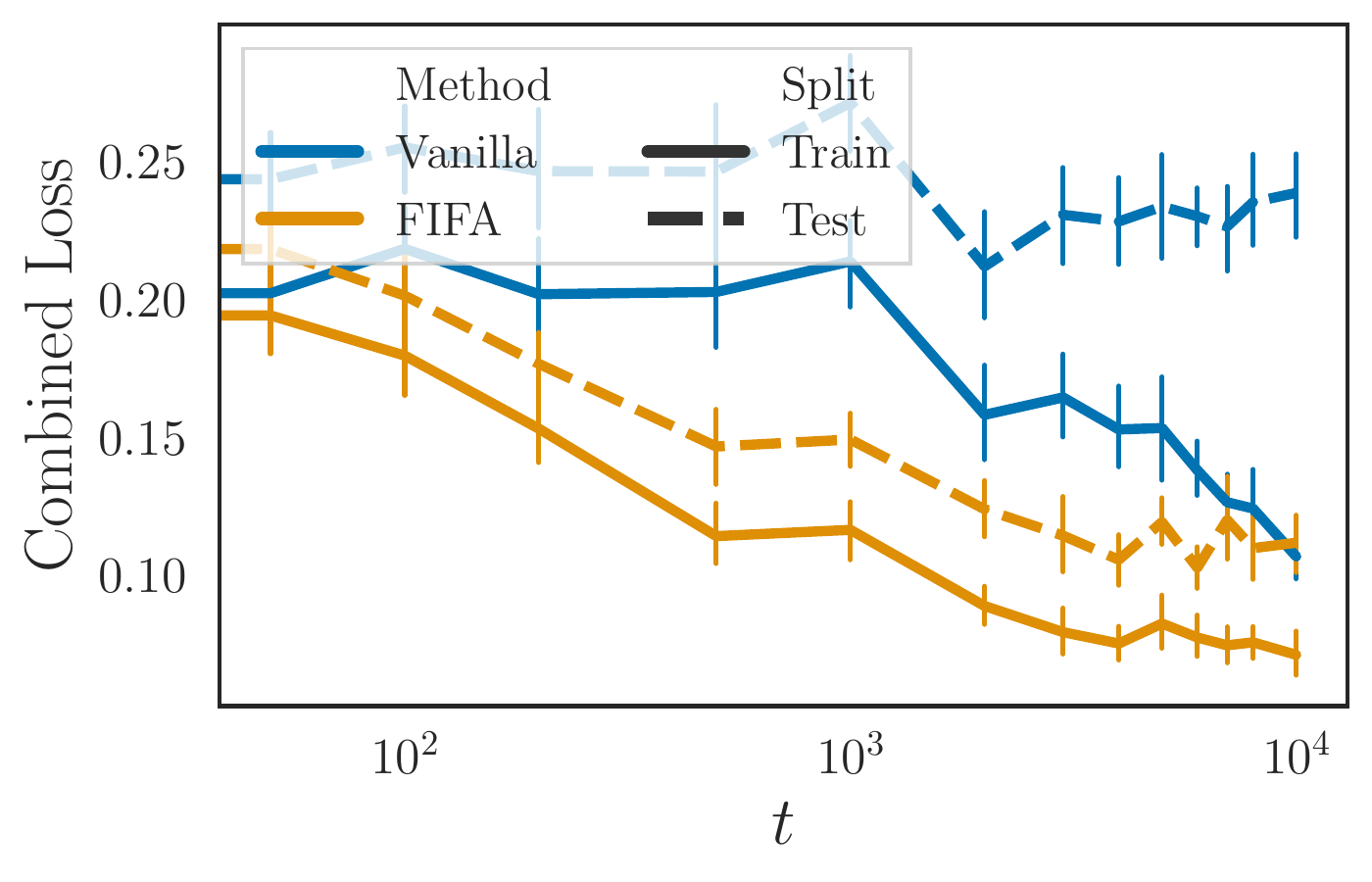}
	    \caption{Combined loss.} \label{fig:celeba:traj:comb}
	\end{subfigure} \\
    \begin{subfigure}[t]{0.45\textwidth}
	    \centering
	    \includegraphics[width=\linewidth]{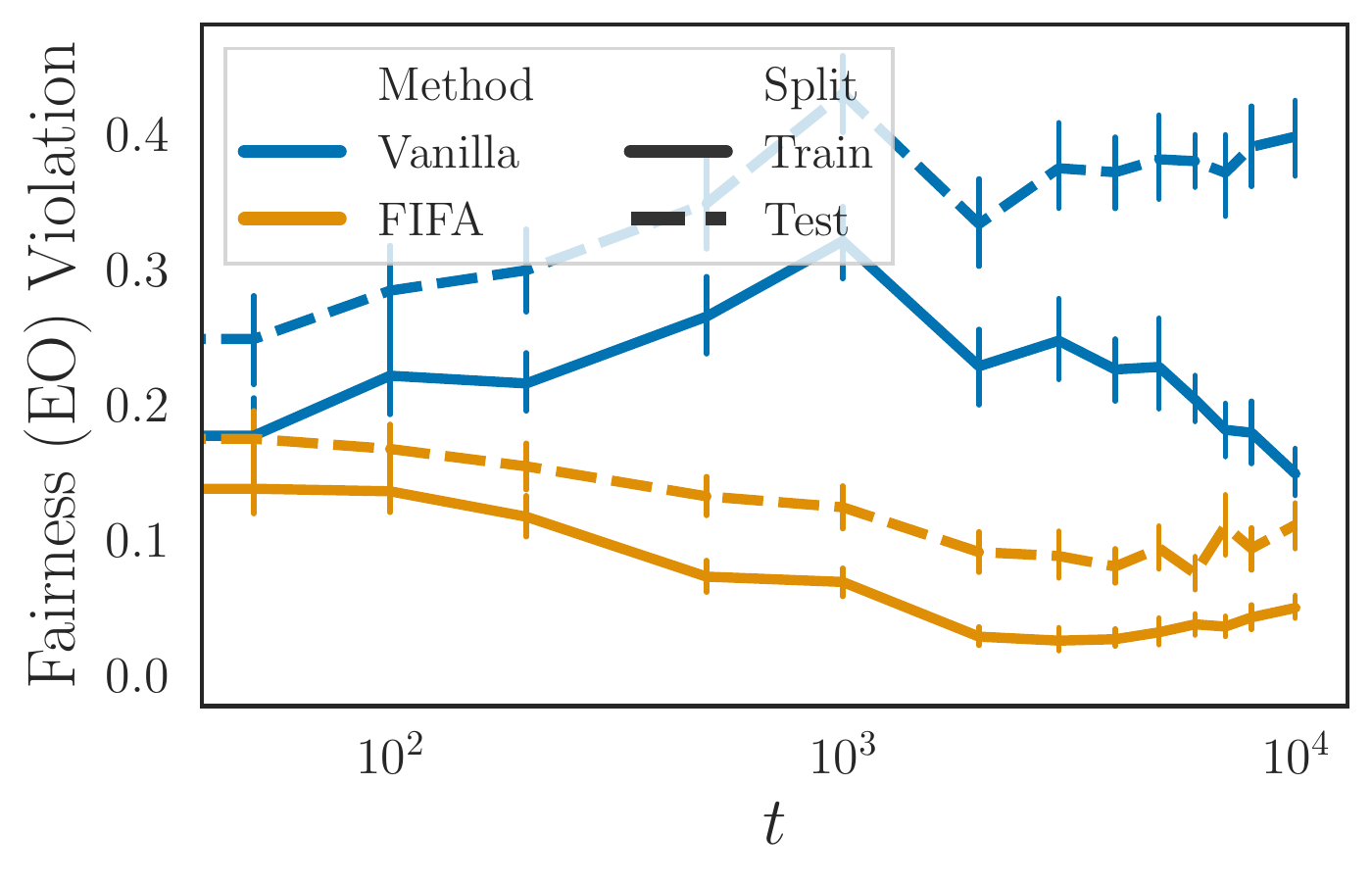}
	    \caption{Fairness violation.} \label{fig:celeba:traj:fv}
	\end{subfigure}~
    \begin{subfigure}[t]{0.45\textwidth}
	    \centering
	    \includegraphics[width=\linewidth]{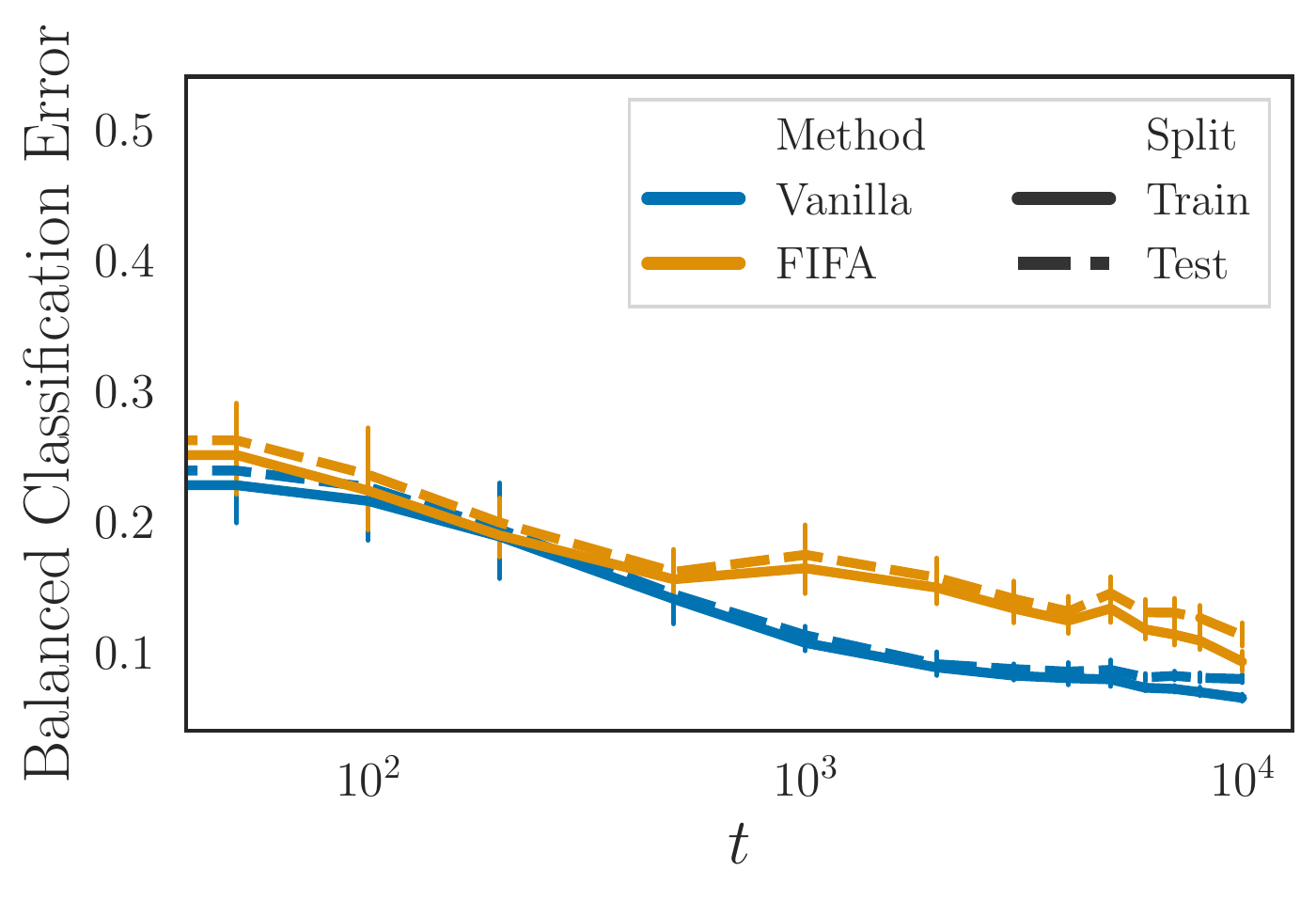}
	    \caption{Balanced error.} \label{fig:celeba:traj:bal}
	\end{subfigure}%
	\caption{Training and test trajectories of different metrics of ResNet-18 on CelebA dataset
	under FIFA and vanilla losses respectively. We note that the generalization performance of vanilla models
	are consistently poor as training time increases,
	suggesting that it is difficult to cultivate an
	early-stopping scheme that might alleviate poor
	fairness generalization.}
	\label{fig:celeba:traj}
	
\end{figure*}
One observation we made in Table 1 is that
the improvements of the generalization of combined loss on CelebA is largely due to the improved generalization
performance on fairness violations. It is natural to wonder whether this behavior suggests that the sweet
spots of generalization performance for balanced error
and fairness violation may not be aligned, i.e.,
there is a difference in training time scales for these two metrics
to reach their optimal generalization.
Furthermore, it is also open that whether one could enforce certain early stopping procedure (e.g., on the combined loss or the fairness violation) such that
the generalizations on vanilla models may be improved.

To explore these two questions, we plot
the trajectories of training and test metrics for FIFA
and vanilla (hyper-parameter chosen to be those corresponding to the best-performing models in Table 1) in Fig.~\ref{fig:celeba:traj}. We observe
that it is difficulty to (i) identify sweet spots of generalization gaps for the vanilla models; and (ii) enforce a reasonable early stopping criterion
that improves the generalization performances thereof.

\subsection{CelebA and the DP constraint}
We presented in Table 1 our main results, CelebA dataset
trained with grid search under EO constraint.
We show in Table~\ref{tab:celeba:dp} the
results on the DP constraints. Here all training configurations are the same as Table 1, except
that we replace the EO constraint by the DP constraint.
For ease of comparison, we also recall the results
on EO in Table~\ref{tab:celeba:dp}.
The observations are similar to those we made for
Table 1, namely, FIFA improves significantly on the combined loss compared with vanilla.

\begin{table*}[h]
    \scriptsize\centering
    \resizebox{\textwidth}{!}{%
\begin{tabular}{llllllllll}
\toprule
                      \multirow{3}{*}{$\epsilon$}       & \multirow{3}{*}{\bf Method}  & \multicolumn{4}{c}{\bf EO} & \multicolumn{4}{c}{\bf DP} \\
                      &         & \multicolumn{2}{c}{\bf Combined Loss} & \multicolumn{2}{c}{\bf Fairness Violation} & \multicolumn{2}{c}{\bf Combined Loss} & \multicolumn{2}{c}{\bf Fairness Violation} \\
                      &         &         Train &   Test &              Train &   Test &         Train &   Test &              Train &  Test \\
\midrule
\multirow{2}{*}{0.01} & FIFA &         7.37\% &  6.71\% &              5.31\% &  2.75\% &         8.65\% &  7.21\% &              4.84\% & 1.45\% \\
                      & Vanilla &         7.14\% & 14.01\% &              6.69\% & 20.29\% &        10.74\% & 10.43\% &              4.72\% & 1.35\% \\
\midrule
\multirow{2}{*}{0.05} & FIFA &         5.46\% &  6.34\% &              2.63\% &  3.29\% &        10.02\% &  9.40\% &              4.73\% & 1.07\% \\
                      & Vanilla &         8.84\% & 13.05\% &              9.45\% & 17.92\% &        12.14\% & 11.60\% &              8.34\% & 5.17\% \\
\midrule
\multirow{2}{*}{0.10} & FIFA &         5.92\% &  6.54\% &              3.11\% &  2.65\% &         9.04\% &  8.32\% &              2.98\% & 0.04\% \\
                      & Vanilla &         8.90\% & 16.71\% &             11.37\% & 26.15\% &        12.13\% & 11.66\% &              9.77\% & 6.83\% \\
\bottomrule
\end{tabular}
    }
    \caption{Grid search with EO and DP constraint on CelebA 
    dataset \citep{liu2015faceattributes} using ResNet-18, best results with respect to test combined loss  among sweeps of hyper-parameters are shown.%
    }
    \label{tab:celeba:dp}
    \vspace{-5pt}
\end{table*}

\subsection{Grid search on AdultIncome (DP and EO)}

\begin{table*}[h]
    \scriptsize\centering
    \resizebox{\textwidth}{!}{%
\begin{tabular}{llrrrrrrrr}
\toprule
                      \multirow{3}{*}{$\epsilon$} &     \multirow{3}{*}{\bf Method}    & \multicolumn{4}{c}{\bf EO} & \multicolumn{4}{c}{\bf DP} \\
                      &         & \multicolumn{2}{c}{\bf Combined Loss} & \multicolumn{2}{c}{\bf Fairness Violation} & \multicolumn{2}{c}{\bf Combined Loss} & \multicolumn{2}{c}{\bf Fairness Violation} \\
                      &         &         Train &      Test &              Train &      Test &         Train &      Test &              Train &     Test \\
\midrule
\multirow{2}{*}{0.01} & FIFA &     14.77618\% & 14.93573\% &           8.53851\% &  8.50086\% &     13.75700\% & 14.05881\% &           0.08609\% & 0.00999\% \\
                      & Vanilla &     16.68724\% & 17.28659\% &          10.39000\% & 10.92794\% &     14.83347\% & 15.09909\% &           3.32436\% & 3.65903\% \\
\midrule
\multirow{2}{*}{0.05} & FIFA &     14.79263\% & 14.91599\% &           8.57436\% &  8.50841\% &     13.74952\% & 14.03440\% &           0.11811\% & 0.01008\% \\
                      & Vanilla &     16.68724\% & 17.28659\% &          10.39000\% & 10.92794\% &     14.83475\% & 15.09909\% &           3.32895\% & 3.65903\% \\
\midrule
\multirow{2}{*}{0.10} & FIFA &     14.70959\% & 14.88935\% &           8.17597\% &  8.16283\% &     13.72278\% & 14.03193\% &           0.11331\% & 0.01011\% \\
                      & Vanilla &     16.68724\% & 17.28659\% &          10.39000\% & 10.92794\% &     14.82782\% & 15.09909\% &           3.31507\% & 3.65903\% \\
\bottomrule
\end{tabular}
    }
\caption{Grid search with EO and DP constraint on AdultIncome 
    dataset using logistic regression, best results with respect to test combined loss  among sweeps of hyper-parameters are shown.%
    }
    \label{tab:income:gs}
\end{table*}

We also give in Table~\ref{tab:income:gs} the grid search results on AdultIncome dataset,
for both EO and DP constraints.
We observe that on this small dataset and small model,
FIFA can also improve generalization performances
for both EO and DP constraints. This further exhibits
the flexibility of the FIFA approach.

\end{document}